\def\notes{0}
\def\isjmlr{0}

\ifnum\isjmlr=1
    \documentclass[anon,12pt]{colt2026} 
\else
    \documentclass[11pt]{article}
\fi

\usepackage{xparse}  %

\let\oldproof\proof
\let\endoldproof\endproof

\ifnum\isjmlr=1
    \providecommand{\qedsymbol}{$\blacksquare$}

    \RenewDocumentEnvironment{proof}{o}{%
      \IfValueTF{#1}{%
        \noindent \textbf{#1}\ \ignorespaces
      }{%
        \oldproof
      }%
    }{%
      \IfValueTF{#1}{%
        \hfill \qedsymbol\par
      }{%
        \endoldproof
      }%
    }
\fi

\ifnum\isjmlr=0
    \usepackage{geometry}
\fi

\usepackage{graphicx, macros} %

\ifnum\isjmlr=1
    \title[Refereed Learning]{\vspace*{1em} %
    Refereed Learning}
\else
    \title{Refereed Learning}
\fi

\ifnum\isjmlr=1
    \usepackage{times}
    \coltauthor{Anonymous submission to COLT 2026}
\else
    \author{Ran Canetti$^\ast$ \and Ephraim Linder$^\ast$ \and Connor Wagaman\thanks{Boston University, \texttt{\{canetti,ejlinder,wagaman\}@bu.edu.}}}
\fi

\date{March 11, 2026}

\begin{document}

\maketitle

\begin{abstract}
We initiate an investigation of learning tasks  in a setting where  the  learner is given access to two  competing provers, only one of which is honest. Specifically,  we consider the power of such learners in assessing purported properties of opaque models.
Following  prior work in complexity theory that considers the  power of competing provers in various settings,  we call this setting \emph{refereed learning}.

After formulating a general definition of refereed learning tasks, we show refereed learning protocols that obtain a level of accuracy that far exceeds what is obtainable at comparable cost without provers, or even with a single prover. We concentrate on the task of choosing the better one out of two \emph{black-box} models, with respect to some ground truth.
While we consider a range of parameters, perhaps our most notable result is in the high-precision range: For all $\varepsilon>0$ and ambient dimension $d$, our learner  makes only one query to the ground truth function,
communicates only $(1+\frac{1}{\varepsilon^2})\cdot\text{poly}(d)$ bits with the provers, and outputs a model whose loss is within a multiplicative factor of $(1+\varepsilon)$ of the best model's loss. Obtaining comparable loss with 
a \emph{single} prover would require the learner to access the ground truth at almost all of the points in the domain.

We also present lower bounds that demonstrate the optimality of our protocols in a number of respects, including prover complexity, number of samples, and need for query access. 
\end{abstract}

    \newpage
    {
    \hypersetup{linkcolor=black}
    \tableofcontents
    }
    \newpage

\ifnum\isjmlr=1
    \clearpage
    \pagenumbering{arabic}
    \setcounter{page}{1}
\fi

\section{Introduction}

Modern machine learning tasks require increasingly large amounts of data and computational power. As a result, model training has shifted from a task that almost anyone can perform on their own  to a task that requires the assistance of external agents that are resource-abundant and have better access to the underlying data. Furthermore, one is often presented with models that purport to approximate some ground truth, but are not accompanied with a rigorous or trustworthy performance guarantee. Moreover, these models are often given as black-boxes---that is, properties of the model are difficult to assess given its description.

This state of affairs naturally raises  the need to verify claims of performance of such  models, with significantly fewer resources than those needed to train comparable models and without fully trusting the parties making the claims.
Verifying such claims appears hard: There is a rich literature on efficient verification of {\em computations} performed by powerful-but-untrusted parties (e.g., \cite{Kilian92,Micali94,GoldwasserKR15}; see \cite{WalfishB15} for a survey); however, such mechanisms appear ill-suited to the task of assessing a model's accuracy, even if the code of  the model is explicitly given.
There is also  a growing body of work on using powerful-but-untrusted intermediaries to learn properties of unknown  ground-truth functions  \cite{GoldwasserRSY21,CanettiK21}, as well as efficient verification of  claims made by powerful provers regarding properties of huge combinatorial objects in general \cite{ErgunKR04,RothblumVW13}.   However, these works mostly focus on verifiable execution of a specific learning  (or property testing) algorithm,  rather than  evaluating the performance of  a given model without knowledge of the process used to  train it.  Perhaps closest to our task  are the works of
\cite{HermanR22,HermanR24-1,HermanR24-2} that allow verifying claims about properties of {\em distributions} that are accessible only via obtaining samples. However, this is still a far cry from assessing performance of a model. 

Some natural properties of ML models can of course be assessed using standard methods. For instance, the loss of a given model w.r.t.\ some sample distribution and loss function can be approximated by computing its empirical loss on a large enough sample. However, this method can be prohibitively costly both  in samples from the ground  truth and in queries to the model.
\cite{GoldwasserRSY21} show a technique for pushing
much of the burden to an external powerful-but-untrusted prover, but even  this method incurs high cost: to obtain an {\em additive} bound  of  $\eta$ on the model's error,  the learner needs to both communicate $\eta^{-2}$ unlabeled samples to the prover and have the prover query and report the ground truth values at these points, and obtain $\eta^{-1}$ labeled samples which are hidden from the prover in order to verify the prover's responses. To see that this is prohibitive, consider a setting where evaluating the ground truth requires performing an expensive physical experiment (e.g., as is required for validating the performance of AlphaFold \cite{Jumper2021}); ideally, assessing the accuracy of a model should only require performing a small constant number of experiments.

We would like to do better---in terms of the error, in terms of the access to the ground truth, and in terms of the communication with the prover. However, as evidenced by lower bounds proven in \cite{GoldwasserRSY21}, this appears hard---at least within the present framing of the problem.

\subsection{This work}

We show that the quality of black-box models can be assessed with significantly better accuracy and with significantly lower cost if the learner can interact with \emph{two powerful and competing provers.}
(The ``competing'' aspect  of the provers can be captured as a zero-sum game where only one of the two provers wins; alternatively we can assume that one of the provers is honest.)
The provers' power can be manifested either in terms of their  computational power, or in access to the ground truth, or in knowledge of the models, or any combination of these.   This  model is a natural  extension of the  \emph{refereed delegation of computation}  model of \cite{FeigeST88,FeigeK97,CanettiRR11,CanettiRR13,KolR14} to our setting.  Following the cue of these works,  we coin the term \textsf{refereed learning} to denote the type of learning performed in this model. (See Section \ref{sec:unrelated} for more discussion on and justification of this model.)

We first  define refereed learning for a general setting. Next  we focus on applying refereed learning to the following specific task, where we showcase the power of  the refereed learning framework
via concrete protocols. 
The learner--verifier\footnote{We use the terms {\em learner} and {\em verifier}  interchangeably. Indeed, the learner now  doubles as a verifier.} is presented with {\em two} candidate models that purport to compute the same ground truth, and is tasked with choosing the model that incurs  the smaller overall loss with respect to some sample distribution and loss function. (The restriction to two candidates is not essential, but it helps make the model more concrete. It also facilitates envisioning each one of the provers as ``trying to promote'' a different one of  the two models.)

The salient parameters we consider are (a) the overall loss of the output model relative to the better of the two competing models; (b) the number of learner queries to the ground truth and samples from it (both labeled and unlabeled ones); (c) the number of learner queries to the candidate models; (d) the computational complexity of the learner; (e)  the computational complexity and  query and sample complexity of the provers.
We first sketch and briefly discuss definitions of refereed learning, then present our results,  and finally discuss   the new tools we develop  to obtain these results.

\paragraph{A motivating example.}
We focus on
a setting where there is a known distribution over unlabeled samples, and where the ground truth is computable, but doing so is costly (e.g., it requires a physical experiment). The task is to compare the accuracy of two ML  models (such as AlphaFold \cite{Jumper2021}) where each model takes proteins as input and returns a prediction of how these proteins interact. The researcher (verifier) can specify a distribution over proteins of interest; however, assessing the quality  of these predictions incurs high cost, even for a given distribution. Specifically, it often requires synthesis and purification of the actual proteins, followed by cryo-EM imaging to confirm the predicted protein interaction. 

\subsubsection{Defining  refereed learning}
\label{sec:our-defs}

We start with  general refereed learning. Consider a learner-verifier $\cV$ and two provers $\cP_0,\cP_1$ that are presented with a ground truth function $f:\zo^d\rightarrow \cY$, a distribution $\cD$ over $\zo^d$, and $k$ models (a.k.a.\ hypotheses) $h_1,\ldots, h_k\in\cH$, where $\cH$ is some family of hypotheses.
We will typically assume that $h_1,\ldots, h_k$ and $f$ are accessed via queries and, depending on the setting, that $\cD$ is accessed either via samples or via queries to its probability mass function, $Q_\cD$, which maps $x\mapsto\Pr_{X\sim\cD}\brackets{X=x}$. 
To measure the learner's performance we use a score function $\cS$ (which assigns a score to each potential output of the learner, with respect to some sample distribution $\cD$, hypotheses $h_1,\ldots, h_k$ and function $f$) and a target function $\cT$.
In the definition below, $\alpha$ and $\eta$ respectively control the multiplicative and additive slack up to which the loss of the returned hypothesis matches that of the best hypothesis. We give protocols with purely multiplicative, purely additive, and mixed error guarantees.

 \begin{definition}[\textbf{Refereed learning, general case (informal)}]
A protocol $[\cP_0,\cP_1,\cV]$ is an \textsf{$(\alpha,\eta,\beta)$-refereed learning }protocol,  with respect to a family $\cH$ of hypotheses, a score function $\cS$  and  target function $\cT$, if for all $b\in\zo$, $h_1,\ldots, h_k\in\cH$,  and $\cP^*_{1-b}$, the learner output $\vb$ satisfies 
$$\Pr[\cS(\vb,\cD,f,h_1,\ldots,h_k) \leq \alpha \cT(\cD,f,h_1,\ldots,h_k)+\eta]\geq 1-\beta.
$$
 \end{definition}

For this general definition, we choose to \emph{minimize} the score $\cS$. This choice is motivated by the learning theory formulation of this problem, where the goal is to select the model (hypothesis) that minimizes a \emph{loss} function quantifying the deviation of the hypothesis from some ground truth.
Concretely, we focus on the setting where the learner selects between two models (hypotheses) $h_0$ and $h_1$, its output bit $\vb$ indicates the selected hypothesis, 
the general scoring function $\cS$ measures the loss of the chosen model with respect to some metric $\ell$ on the domain  $\cY$, and the target function $\cT$ is the smaller of the losses of $h_0$ and $h_1$.
(Additionally, we are primarily interested in the following loss: for some hypothesis $h$, sample  distribution $\cD$, and ground truth $f$, let $\cL_{\cD}\paren{f,h \mid \ell} = \Ex_{x\sim \cD}\brackets{\ell\bparen{f(x), h(x)}}$.)
In this paper, we focus on refereed learning protocols satisfying the following definition.

\begin{definition}[\textbf{Refereed learning, loss minimization (informal)}]
A protocol $[\cP_0,\cP_1,\cV]$ is an \textsf{$(\alpha,\eta,\beta)$-refereed learning protocol for loss minimization,}  with respect to a family $\cH$ of hypotheses and metric $\ell$, if for all $b\in\zo$, $h_0, h_1\in\cH$, and  $\cP^*_{1-b}$,
the learner output $\vb$ satisfies
$$\Pr\brackets{\cL_\cD(f,h_{\vb}\mid \ell) \leq \alpha \min_{s\in\zo}\cL_\cD(f,h_{s}
\mid \ell)+\eta}\geq 1-\beta.
$$
\end{definition}

\paragraph{Bounding the expected score and loss.} An alternative formulation would instead bound the {\em expected}  score  (respectively, loss).  That is,  the requirement  would be that 
$\Ex[\cS(\vb,\cD,f,h_1,\ldots, h_k))]\leq \alpha \cT(\cD,f,h_1,\ldots, h_k)+\eta$ (in the general case), or  
$\Ex[\cL_\cD(f,h_\vb\mid \ell)]\leq \alpha\min_{s\in\zo} \cL_\cD(f,h_s \mid \ell)+\eta$ (in the case of minimizing the loss).
While the two formulations are incomparable in general,  our protocols satisfy both with similar parameters. 

\paragraph{On strategic provers.} The above definition posits that at least one of the provers follows the protocol. We note, however,  that  refereed learning protocols often seem  to preserve their guarantees even when both provers are strategic with opposing goals.  Some supporting evidence for the implication is  the use of protocols developed in the refereed delegation of computation model in real-world applications where truth-telling is economically incentivized. Similar phenomena are manifested in the context of  debate systems. See more details in \Cref{sec:unrelated}.

\subsubsection{Our results}
\label{sec:results}
 
\paragraph{Protocols.} We give protocols for the multiplicative, additive, and mixed error settings. In the additive error setting,
we show how to use the two provers to obtain  additive error similar to that obtained by the \cite{GoldwasserRSY21} protocol mentioned above, while  significantly reducing the learner's interaction with both the ground truth and the models: this interaction now consists of   only a {\em single query.}

We then show, via simple extension, that even the provers can use  significantly fewer queries  
at the cost of obtaining an error bound that is both additive and multiplicative. %
Specifically, 
to obtain additive loss at most $\eta$ and multiplicative loss at most $1+\eps$, our learner only makes a single query to either the ground truth or one of the models, draws $\paren{1+\frac 1{\eps^2}}\cdot \frac 1\eta$ \emph{unlabeled} sample points from the underlying distribution, and has the provers query each model on all of the unlabeled sample points, and query the ground truth on $1+\frac1{\eps^2}$ of them.

We focus primarily on designing protocols for the low-loss setting, which turns out to be significantly more challenging. Here we would like to guarantee that the learner makes the right choice even when the models' losses are close to each other, up to a  {\em  multiplicative} factor of $1+\eps$ for an arbitrarily small $\eps>0$. Indeed, in such a setting, the number of samples needed to even observe the difference can be close to the entire sample space. Still, determining which of the two  competing models incurs smaller loss can have significant ramifications in applications that require high precision (e.g., using ML models for medical predictions based on imaging or other multi-dimensional  measurements,  or for financial applications where even tiny error margins become significant over time).\footnote{
The purely multiplicative error setting can also be a better fit when working with competing provers.
To see this, consider the following toy example. Suppose the verifier can only approximate provers' additive errors up to some threshold $\eta$ (and is unable to distinguish provers' losses that differ additively by at most $\eta$). Here each prover is only incentivized to learn up to loss $\eta$ if possible, but no better, since its hypothesis will be accepted with probability (at least) $\frac 12$ once it learns to loss at most $\eta$. Now suppose the verifier can approximate the ratio of the provers' error up to $(1\pm \eps)$. Here, a prover who can learn to loss 0 will do so (assuming sufficient incentives---e.g., no cost for additional computation).
}

We first concentrate on the zero-one metric\footnote{Define the \emph{zero-one metric} $\lzo$ by $\lzo(y,y')=1$ if $y\neq y'$ and $\lzo(y,y)=0$.  Let the \emph{zero-one loss} between functions $f$ and $h$ w.r.t.\ distribution $\cD$ be $\Ex_{x\sim\cD}\brackets{\lzo(f(x),h(x))}=\Pr\brackets{f(x)\neq h(x)}$.}
on $\cY$ and the uniform underlying distribution. In this setting, we design a protocol where the learner is guaranteed (except with some arbitrarily small constant probability) to select a model whose loss is at most a multiplicative factor of $(1+\eps)$ worse than the better model's. Moreover, the learner in this protocol is efficient: it makes a single query %
to the ground truth function $f:\zo^d\to\cY$ and has $(1+\frac1{\epsilon^2}) \poly(d)$  communication with the provers.

We then extend these results to handle arbitrary metric loss functions and arbitrary sample distributions.
Our results also take into account the cost of obtaining the desired level of numerical precision. The protocol for this more general setting is guaranteed to select a model whose loss is at most a multiplicative factor of $3+\eps$ worse than the best model. Moreover, the learner makes  a single query to either $f$, the two models, or the distribution,  has the same runtime as before, and only incurs a small cost of $\lambda\cdot\poly d$ in communication with the provers,
where $\lambda$ roughly bounds the bit precision of the numerical representations used for sending values between parties in the protocol.
We also show how to handle arbitrary precision at the cost of incurring a tiny additive error, with essentially no overhead in communication and runtime. 

The prover complexity in the protocols
with purely multiplicative error
may well depend on the hypothesis class, on the competing models,  and on the prover's knowledge of both. In  the extreme case where the prover has no apriori knowledge of the ground truth or the models, an exponential number of queries to the model is needed to follow our protocol. We demonstrate that this exponential overhead is inherent for a general solution. Despite this strong lower bound, we also demonstrate a setting  where the provers can use some apriori knowledge on the models to gain computational efficiency.

\paragraph{Lower bounds.}
We complement our protocols by establishing lower bounds that justify some of the complexity parameters of our protocols. 
We show that in any refereed learning protocol where the learner obtains additive error at most $\eta$, and either (a) accesses the ground truth only via labeled samples taken from the given distribution,  or else (b) has no knowledge of the underlying distribution of samples other than the samples obtained, the number of samples  that the learner obtains from the ground truth must be at least $\frac 1\eta$.
In other words, unless the learner both queries  the ground truth and also obtains additional information on the underlying distribution (say, the value of its PMF at certain points), $\eta\rightarrow 0$ is unattainable.

We also show that the prover's exponential runtime  in any general-purpose refereed learning protocol with purely multiplicative error is inherent. The argument for the case where the models can be accessed only as black boxes is straightforward and unconditional.  We then demonstrate the need for exponential computational power even for a general solution where the provers have a full whitebox view of the models. Specifically, we show that a refereed learning protocol with a purely multiplicative error guarantee 
can be used to solve computational problems (such as 3SAT) that are assumed to be exponentially hard. This means that, assuming the hardness of these problems, a refereed learning protocol cannot  in general be executed in polynomial time, even in the case where all parties  have full white-box access to the model.

\subsubsection{Our techniques}
\label{sec:techniques}

\paragraph{Protocols for purely additive and mixed errors.}
As a warm-up, consider
the setting
with additive error $\eta>0$ and the zero-one metric on $\cY$. (We describe these protocols in more detail in \Cref{sec:add-mixed-error}.) 
The natural way for the learner to bound the additive error by $\eta$ without using the provers is to draw $O(1/\eta^2)$ samples from $\cD$, query  $f,h_0,h_1$ at these points,  and pick the hypothesis with the smaller empirical loss.   When multiplicative loss $\alpha=1+\eps$ with $\eps>0$ is also allowed then the learner can choose to draw only $(1+\frac{1}{\eps^2})\cdot\frac1\eta$ samples from $\cD$, find the set $S$ of  samples on which $h_0$ and $h_1$ disagree, and then pick the hypothesis with the smaller empirical loss {\em over the samples in $S$.} 

The learner  can then use the provers as follows:  instead of  directly querying $f,h_0,h_1$  on  the sampled points, the learner can send the sampled points to the provers and have the provers make the queries and report back the obtained values.  If the provers disagree on any returned value, the learner picks  {\em one} such value, makes  the query itself, and proceeds with the values provided by the prover that reported the correct value.

\paragraph{Protocols for purely multiplicative error.}
Now consider
the more challenging setting where only  multiplicative error is allowed, i.e.,  $\eta=0$. Directly extending the above protocol from the mixed error case would be meaningless,  since the number of samples required exceeds the domain size when $\eta$ approaches 0.  An alternative approach would be to compute the empirical loss of the two hypotheses over a {\em sufficiently large sample} from  the ``disagreement set'' $S = \sets{x \mid h_0(x)\neq h_1(x)}$;  however, if  the set $S$  is sparse then the learner cannot efficiently sample  from it.  Furthermore, it is not immediately clear how the learner can use the provers to provide it with  correctly distributed samples from $S$.  In particular, the above method of settling  discrepancies between the provers regarding the value of either one of $h_0,h_1,f$ at a given point does not seem to be useful for  the purpose of obtaining a random sample from $S$.
To get around this issue  we devise a {\em certifiable sampling} protocol, described below, that allows the learner to obtain samples from $S$ that are guaranteed to be correctly distributed. Given this protocol, the learner simply picks the model with the smaller empirical loss over $O(1+\frac1{\eps^2})$ random samples from $S$. Finally,  the learner can offload all queries but one to the provers, as described earlier.
The certifiable sampling protocol makes heavy use of the \emph{certifiable sum} protocol.

\paragraph{The certifiable sample protocol.}
This protocol allows the verifier to efficiently obtain $m$ samples from an arbitrary distribution $\cD$ over $\zo^d$---even when the support of $\cD$ is \emph{both exponentially large and exponentially sparse}. At a high level, the protocol does this through inverse CDF sampling facilitated by certifiable sum (described next). That is, the verifier specifies some total ordering on $\cD$ and then samples a uniformly random value $p\in [0,1]$. The provers then return the element $x'\in\cD$ such that $\sum_{x\prec x'}\cD(x) < p\leq \sum_{x\preceq x'}\cD(x)$. The verifier then uses certifiable sum to confirm that these inequalities hold for the returned value $x'$.
See \Cref{lem:cert_sample} for a formal statement of the certifiable sample protocol, and \Cref{thm:rlp_zo} for a formal statement of the resulting refereed learning protocol.

\paragraph{The certifiable sum protocol.}
This protocol 
allows the learner, assisted by provers, to  evaluate quantities of the form $s=\sum_{x\in\zo^d}t(x)$, given only query access to the function $t$, in time that is polynomial in $d$.

The protocol has two symmetric stages (one for each prover); for clarity, we just describe one stage.
At a high level, each stage of the protocol starts by having  one prover  make a claim about the total sum $s$ along with a claim about the sums $s_0,s_1$ on two disjoint halves of the domain.
The learner then asks the other prover to identify a half of the domain on which the other prover is lying (if there is such a half).
This process continues recursively, for $d$ rounds, until the learner is left with a single point $x^*$ and the prover's claim that $t(x^*)=y$. The learner can then check this claim in a single query to $t$.
The key observation is that if a malicious prover misreports the value of the sum, then it must incorrectly report the value of the sum on at least one half of the domain. Thus, if the second prover is following the protocol, then the lying  prover is bound to be caught in one of the $d$ rounds.  Indeed,  if the prover ever misreports the sum on one half of the domain, that prover will necessarily lie about the sum of $t$ on (at least) one half of the remaining domain, and the final query of $t(x^*)$ will reveal that the cheating prover is lying; an honest prover's claims, by contrast, will always be found true. See \Cref{lem:cert_sum} for the full description.

\begin{theorem}[Refereeed learning for zero-one loss (\Cref{thm:rlp_zo}, informal)]
\label{thm:rlp_zo-inf}
    There exists a protocol $[\cP_0, \cP_1, \cV]$ that, for all $\eps,\beta> 0$ and distributions $\cD$ on $\zo^d$, is a $(1+\eps, 0, \beta)$-refereed learning protocol for $f$ and $\cD$ with respect to $\lzo$. The protocol has communication complexity and verifier runtime $\wt O\paren{(1+\frac{1}{\eps})^2\cdot \poly d \cdot \log\frac{1}{\beta}}$, and the verifier makes 1 query to $f$.
\end{theorem}

\paragraph{Extending to general metric loss functions.} 
Next we deal with the case of a general  (i.e., not zero-one) metric  $\ell$ on the range $\cY$.  This case is more challenging since different points in $\cY$ can have very different contributions to the overall loss. In particular,  there may be a single point $x^*$ with $\ell(h_1(x^*),f(x^*))\gg \ell(h_0(x^*),f(x^*))$, and thus a naive learner which, as in the zero-one case, samples from the set $S=\sets{x\mid \ell(h_0(x),h_1(x))>0}$ and outputs the hypothesis with better loss on the sample, will require many samples before obtaining the point $x^*$ and thus may select the worse model.  

To circumvent this issue, we define a rescaled version of $\cD$, denoted $\drescale{h_0,h_1}$, that places more mass on points $x$ where $\ell\bparen{h_0(x),h_1(x)}$ is large. By the triangle inequality, if $\ell\bparen{h_0(x), h_1(x)}$ is large, then either $\ell\bparen{h_0(x), f(x)}$ or $\ell\bparen{h_1(x), f(x)}$ must be large as well. Roughly, this technique resolves the above difficulty of finding a higher-loss point $x^*$ since if there is such a point,
then the rescaled distribution will place proportionally more mass on $x^*$. More generally, we show that under the rescaled distribution $\drescale{h_0,h_1}$, with high probability the worse hypothesis accounts for more than half of the combined empirical loss of $h_0$ and $h_1$ on $f$, computed over only $O\paren{1+\frac{1}{\eps^2}}$ samples. Leveraging the techniques we develop for the zero-one case, we show that the learner, with the help of the provers, can efficiently provide itself with query access to the probability mass function of $\drescale{h_0,h_1}$, and thus can execute the certifiable sample protocol  to efficiently generate samples from $\drescale{h_0,h_1}$. The full treatment appears in  \Cref{sec:rlp-gen}. Roughly speaking, the statement in \Cref{thm:rlp_zo-inf} holds for the general setting, except we obtain a $(3+\eps, 0, \beta)$-refereed learning protocol.

\begin{theorem}[Refereeed learning for general loss functions (\Cref{thm:rlp_gen}, informal)]
    There exists a protocol $[\cP_0, \cP_1, \cV]$ that, for all $\eps,\beta> 0$ and distributions $\cD$ on $\zo^d$, is a $(3+\eps, 0, \beta)$-refereed learning protocol for $f$ and $\cD$ with respect to a specified metric $\ell$. The protocol has communication complexity and verifier runtime $\wt O\paren{(1+\frac{1}{\eps})^2\cdot \poly d \cdot \log\frac{1}{\beta}}$, and the verifier makes 1 query to $f$.
\end{theorem}

\paragraph{Efficient refereed learning for juntas.} We complement the above general-purpose refereed learning protocols, where the provers need exponential time, by demonstrating a refereed learning protocol for a natural learning task where the provers can be efficient. Specifically, we consider the case where $h_0$ and $h_1$ are promised to be \emph{$j$-juntas} (i.e., Boolean functions that each depend only on some  set of $j$ input coordinates), and the  active index sets $J_0$ and $J_1$ of $h_0$ and $h_1$ are given as input to all parties. In this setting, we show that the provers can be implemented efficiently. The idea is the following: In the general case, the task that determines the prover runtime  is computing the set $S=\sets{x\mid h_0(x)\neq h_1(x)}$; when $h_0$ and $h_1$ are promised to be juntas with $j\approx \log d$ active indices, the provers can compute $S$ in time $\poly d$. Since the distribution is uniform, the certifiable sample protocol used in the general case can also be implemented efficiently, and thus the provers can be made to run in time $\poly d$. See \Cref{prop:juntas} for a formal statement and construction of the protocol.

\paragraph{Lower bounds.}
As described earlier, we show several  impossibility results, each demonstrating the optimality of a different aspect of the protocols.
The first result (\Cref{thm:lower_bound_samples}) demonstrate that  without query access to the ground truth $f$, a learner would in general need a prohibitive number of queries. 
The argument is straightforward: fix $\sets{h_0,h_1}$, sample $b\sim\zo$ and let $f\gets h_b$. Consider the cheating prover that executes the honest protocol, except it ``pretends'' that $f=h_{1-b}$. As long as the learner does not obtain any sample $x$ with $f(x)\neq h_{1-b}(x)$, it cannot refute the malicious prover's claim that $h_{1-b}$ has zero loss, and thus cannot determine if it should accept $h_0$ or $h_1$. The proof that query access to $Q_\cD$, the PMF of $\cD$, is necessary (\Cref{thm:dist_knowledge}) follows a similar outline. 

Turning to  the complexity of the provers,  we first observe that, when the provers only have black-box access to $h_0$ and $h_1$, the provers may need to query  $h_0$ and $h_1$ at $\Omega(2^d)$ points to find the hypothesis with better loss, up to a multiplicative factor. (To see this, for all $z\in\zo^d$ let $h_z=\Ind\brackets{x=z}$, and sample hypotheses $z,z'\sim\zo^d$ and ground truth  $f\sim\sets{h_z,h_{z'}}$. Any algorithm $\cA$ which gets query access to $h_z,h_{z'}$, and $f$ cannot distinguish whether $f=h_z$ or $f=h_{z'}$ until it queries either $z$ or $z'$. Since these are uniformly random points, $\cA$ must make $\Omega(2^d)$ queries.) In \Cref{thm:rlp_runtime_lb} we then extend this bound to the case where the provers are given an explicit description of $h_0$ and $h_1$. This bound 
proceeds by reduction from 3-SAT:  Any refereed learning protocol which guarantees  any purely multiplicative  bound on the loss of the computed hypothesis  can be used to distinguish satisfiable formulas from unsatisfiable ones.

\subsection{Related work}
\label{sec:unrelated}

This work  combines ideas,  formalisms, and techniques from a number of different areas. Here we briefly review some of the main works that inspired the present one, as well as works that may appear related but differ in some substantial ways. 

The idea of considering the computational power of a model that involves a weak verifier and two or more provers, one of which is assumed to be honest, goes back to the works of Feige et al.~\cite{FeigeST88,FeigeK97}, who also observe that the  provers can be viewed as {\em competing.}\footnote{
Indeed, with sufficiently large rewards to provers whose proposals are adopted and who demonstrate the other prover's inconsistencies, as well as harsh penalties to provers who demonstrate inconsistency of the other prover, the provers are incentivized to (a) always provide the best possible model and (b) never provide a demonstrable inconsistency---even without assuming honesty of either prover. 
} 
Later, \cite{CanettiRR11,CanettiRR13} and \cite{KolR14} have demonstrated several  protocols for delegating arbitrary computations  to untrusted servers in this model.  These {\em refereed delegation} protocols are both  simpler and significantly more efficient than ones designed for a single untrusted prover.  In fact, one of the protocols in \cite{CanettiRR13} is currently in commercial use  within an application where provers are  competing strategic agents, and the protocol is used to incentivize the agents to be truthful \cite{ArunATQKBEFB25}.
Works in the two-server model, such as Prio of \cite{Corrigan-GibbsB17}, also often assume honesty from at least one server, though that line of work is closer to threshold MPC than to the refereed computation literature.

We note, however, that  known refereed delegation protocols do not appear to be directly applicable to our setting.  In particular, these protocols are geared towards verifying fully specified deterministic computations with inputs that are readable by the verifier in full. In contrast,  in our setting  the learner is presented with a black box model whose code is unknown and whose sample space is potentially huge. 

The works of
\cite{ErgunKR04,RothblumVW13} consider a weak verifier that uses the power of a {\em single} untrusted prover to decide  whether some huge mathematical object, which is accessible to the verifier only via queries,  has some claimed property, or else is far from having the property.
\cite{HermanR22,HermanR24-1,HermanR24-2} consider the case where the object in question is a distribution, and the verifier's  access to the distribution is via obtaining samples.
\cite{GoldwasserRSY21} consider a (single) prover that wishes to convince a suspicious verifier that a given concept has some desirable properties.  %
 In all, to the best of our knowledge, this is the first work that considers the power of the two-prover model in the context of learning, testing, or verifying the properties of black-box objects,  models being a special case. 

 A related and very vibrant area of research is that of {\em debate systems} where a panel of competing AI agents debate in order to help a weak referee (either a human or another AI agent) obtain a meaningful decision, often with respect to AI safety and alignment. See, e.g.,  \cite{IrvingCA18,GuoCWCPCWZ24}.  However, both the methods and the specific goals in these works are very different from the ones here.

\subsection{Organization}

\Cref{sec:refereed_learning} introduces a framework for refereed learning and provides a formal definition of refereed learning protocols. \Cref{sec:tools}  develops several key tools which are used to construct refereed learning protocols. \Cref{sec:rlp} leverages these tools to construct refereed learning protocols for the zero-one loss (\Cref{thm:rlp_zo}) and for general metric loss functions (\Cref{thm:rlp_gen}).  \Cref{sec:lower_bounds} proves several lower bounds which justify the learner's access model and the runtime of the provers. We conclude with \Cref{sec:app_ext,sec:add-mixed-error}. \Cref{sec:app_ext}  extends the earlier protocols to the setting where the distribution and loss are measured to arbitrary precision, as well as an application to junta functions where the provers can be implemented efficiently. \Cref{sec:add-mixed-error} presents some simple protocols for the additive and additive/multiplicative error settings.

\section{Framework for refereed learning}
\label{sec:refereed_learning}

In this section we formally define a \emph{refereed learning protocol}. First, we provide a definition in terms of a ``score'' and ``target'' function, and second, we provide a definition for the special case where the score and target correspond to loss minimization.

Throughout this paper we use the following standard notation: for a protocol $\brackets{\cP_0,\cP_1,\cV}$, let $\brackets{\cP_0(A),\cP_1(B),\cV(C)}(D)$ be a random variable denoting the output of $\brackets{\cP_0,\cP_1,\cV}$ when prover $\cP_0$ has input $A$, prover $\cP_1$ has input $B$, learner-verifier $\cV$ has input $C$, and all have input $D$. We define the \emph{communication complexity} of a protocol $\brackets{\cP_0,\cP_1,\cV}$ as the number of bits sent between $\cP_0$, $\cP_1$, and $\cV$. Additionally, for an algorithm $\cA$ and function $f$, let $\cA^f$ denote $\cA$ with query access to $f$---that is, $\cA$ can specify a query $x$ and receive response $f(x)$. The \emph{query complexity} of $\cA^f$ is the number of queries made by $\cA$ to $f$. For a distribution $\cD$, let $\cA^{\cD}$ denote $\cA$ with access to samples drawn from $\cD$. The \emph{sample complexity} of $\cA^{\cD}$ is the number of samples $\cA$ draws from $\cD$. We say a prover $\cP$ is \emph{honest} if it runs the algorithm that is specified by the protocol; a \emph{malicious} prover (denoted $\cP^*$) may deviate from the algorithm specified by the protocol.

In order to define refereed learning, we first define a score and a target. A \emph{score} $\cS$ (parameterized by $k,d\in\N$ and a set $\cY$) is a function that sends tuples $(\vb,f,h_1,\dots,h_k,\cD)\mapsto \R$, where $\vb$ is an output of the learner, $f,h_1,\dots,h_k$ are functions $\zo^d\to\cY$ for some fixed range $\cY$, and $\cD$ is a distribution over $\zo^d$. Additionally, a \emph{target} $\cT$  (parametrized by $k,d\in\N$ and a set $\cY$) is a function that sends tuples $(f,h_1,\dots,h_k,\cD)\mapsto \R$. 

In order to encode various ways in which the parties can access the ground truth, the sample distribution, and the hypotheses,  we formalize an {\em access model} as three oracles, one for each prover and one for the verifier. An oracle allows some sample requests and queries (namely,  it returns either a sample or the value of the relevant function applied to the query,  as appropriate) while disallowing others. 
We let  $\cA^{\cO(f,h_1,\ldots, h_k,\cD)}$ denote algorithm $\cA$ with access to $f,h_1,\dots,h_k$ and $\cD$,  controlled by oracle $\cO$.

\begin{definition}[Refereed learning protocol---general case]
\label{def:rlp_abstract}
 Fix score $\cS$ and target $\cT$ with respect to parameters $k,d\in\N$ and range $\cY$.  Let $\cH\subset\sets{h:\zo^d\to\cY}$ and $\mathds D\subset \sets{\cD\mid \supp(\cD)\subset\zo^d}$.
 Fix oracle access models $\acc_0,\acc_1$ and $\acc_\cV$, slack parameters $\alpha\geq 1$ and $\eta\geq 0$, and soundness error $\beta\geq 0$.
 A protocol $\brackets{\cP_0,\cP_1,\cV}$ is a \emph{$(k,\alpha,\eta,\beta)$-refereed learning protocol (RLP) for $\cH$ and $\mathds D$ with respect to $\cS,\cT$, and oracles $\acc_0,\acc_1$ and $\acc_\cV$}, if for all distributions $\cD\in\mathds D$, (possibly randomized) functions $f:\zo^d\to\cY$ and $h_1,\dots,h_k\in\cH$, the following holds: 
     \begin{itemize}
        \item For all $b\in\zo$ and $\cP^*_{1-b}$, the output $\vb\gets \brackets{\cP^{\acc_b(f,h_1,\ldots,h_k,\cD)}_b, \cP^*_{1-b}, \cV^{\acc_\cV(f,h_1,\dots,h_k,\cD)}}$ satisfies         
    \[
        \Pr\Bigl[\cS(\vb, f, h_1,\dots,h_k,\cD) \leq \alpha\cdot \cT(f,h_1,\dots,h_k,\cD) + \eta\Bigr] \geq 1-\beta,
    \]
    where the randomness is over the coins of the verifier, the honest prover, and the oracles. 
    \end{itemize}
\end{definition}

The definition above is quite general and can be applied to settings beyond learning. In this work, we focus on using the refereed setting for learning and adopt the following, more concrete definition.
As compared to \Cref{def:rlp_abstract}, in \Cref{def:rlp_concrete} we replace the score and target with a metric loss function $\cL$ and only consider the case of $k=2$. 

\begin{definition}[Metric loss function, zero-one metric]
    \label{def:loss_function}
    Fix a range $\cY$ with metric\footnote{A metric $\ell$ satisfies non-negativity (i.e., $\ell(y,y')>0$ if and only if $y\neq y'$), symmetry, and the triangle inequality.} $\ell:\cY\times\cY\to\R$. Additionally, for all $d\in\N$ and all functions $f,h:\zo^d\to\cY$ and distributions $\cD$ over $\zo^d$, define the \emph{metric loss between $f$ and $h$ with respect to $\ell$ and $\cD$} as $\cL_\cD\paren{f,h\mid \ell} = \Ex_{x\sim \cD}\brackets{\ell\bparen{f(x), h(x)}}$. We will omit the dependence on $\ell$ when it is clear from context. Additionally, we define the \emph{zero-one} metric $\lzo$ by $\lzo(y,y')=1$ if $y\neq y'$ and $\lzo(y,y)=0$.
\end{definition}

\begin{definition}[Refereed learning protocol---loss minimization]
\label{def:rlp_concrete}
Fix range $\cY$ with metric $\ell$, dimension $d\in\N$, and $\cH\subset\sets{h:\zo^d\to\cY}$ and $\mathds D\subset \sets{\cD\mid \supp(\cD)\subset\zo^d}$. Fix oracle access models $\acc_0,\acc_1$ and $\acc_\cV$, slack parameters $\alpha\geq 1$ and $\eta\geq 0$, and soundness error $\beta\geq 0$.
A protocol $\brackets{\cP_0,\cP_1,\cV}$ is a \emph{$(\alpha,\eta,\beta)$-refereed learning protocol (RLP) for $\cH$ and $\mathds D$ with respect to $\ell$ and oracles $\acc_0,\acc_1$ and $\acc_\cV$}, if for all distributions $\cD\in\mathds D$, functions $f:\zo^d\to\cY$ and $h_0,h_1\in\cH$, the following holds: 
     \begin{itemize}
        \item For all $b\in\zo$ and $\cP^*_{1-b}$, the bit $\vb\gets \brackets{\cP^{\acc_b(f,h_0,h_1,\cD)}_b, \cP^*_{1-b}, \cV^{\acc_\cV(f,h_0,h_1,\cD)}}$ satisfies         
    \[
        \Pr\Bigl[\cL_{\cD}(h_\vb,f \mid \ell) \leq \alpha\cdot \min_{s\in\zo}\cL_{\cD}(h_s, f \mid \ell) + \eta\Bigr] \geq 1-\beta,
    \]
    where the randomness is over the coins of the verifier, the honest prover, and the oracles. 
    \end{itemize}
\end{definition}

\Cref{def:rlp_concrete} allows the distribution, metric, and functions to take on arbitrary real values. To handle issues with describing and sending arbitrary real-values, we focus on the setting where the distribution $\cD$ and the metric $\ell$ can be represented exactly with the following discretization.
Specifically, where we let $\lamprec := (-2^{\lambda},2^{\lambda})\cap 2^{-\lambda}\cdot \Z$ denote the set of multiples of $2^{-\lambda}$ bounded in magnitude by $2^{\lambda}$, we say a function $f$ is \emph{$\lambda$-precise} if its range is $\lamprec$.
We say a distribution $\cD$ is $\lambda$-precise if its PMF $Q_\cD$ is $\lambda$-precise. Additionally, let $\mathds D_{d,\lambda}$ denote the set of $\lambda$-precise distributions over $\zo^d$ (when $d$ is clear from context we write $\mathds D_\lambda$ instead of $\mathds D_{d,\lambda}$).

Additionally, in order to succinctly refer to the set of all functions $\zo^d\to\cY$ and all distributions over $\cD$, we define the following families:

\begin{definition}[Families $\fF$ and $\fD$]
    \label{def:class_all} For all $d\in\N$ and sets $\cY$, let $\fF_{d,\cY}=\sets{f:\zo^d\to\cY}$ and $\fD_d=\sets{\cD \mid \supp(\cD)\subset\zo^d}$. When $d$ and $\cY$ are clear from context we write $\fF$ and $\fD$. 
\end{definition}

\section{Tools for refereed learning}
\label{sec:tools}

In this section we develop several key tools for designing refereed learning protocols.

\subsection{Certifiable sample and certifiable sum}
\label{sec:cert_sample-sum}

Our first tool, certifiable sample (see \Cref{lem:cert_sample}), allows the verifier to efficiently sample from a distribution that is close to $\cD$ given query access to its \emph{probability mass function} $Q_\cD$ (defined by $Q_{\cD}(x) = \Pr_{X\sim \cD}\brackets{X=x}$).
For all distributions $P$ and $Q$ over domain $\cX$, the \emph{total-variation distance} (TV) distance between $P$ and $Q$ is $\dtv(P,Q) = \sup_{A\subset\cX} \abs{P(A) - Q(A)}$.
For a set $S$, a total ordering $\prec$, and an algorithm $\cA$, let $\cA^{S,\prec}$ denote $\cA$ with
membership query access to $S$ and query access to $\Ind\brackets{\cdot \prec\cdot }$ (the function that takes as input $(x,x')$ and returns $1$ if $x\prec x'$ and $0$ otherwise).

\begin{lemma}[Certifiable sample]
\label{lem:cert_sample}
    \Cref{prot:NEW-cert_sample} has the property that
    for all $d\in\N$, distributions $\cD$ over $\zo^d$ with probability mass function $Q_\cD$,
    distance $\delta\in (0,1)$,
    and sample size $m\in\N$, there exists a distribution $\wh \cD$ with $\dtv\bparen{\wh\cD, \cD}\leq \delta$ such that:
    \begin{enumerate}
        \item\label{lem:cert_sample-1} For all $b\in\zo$ and $\cP^*_{1-b}$, the output $\brackets{\cP^{Q_\cD,\prec}_b, \cP^*_{1-b},\cV^{Q_\cD,\prec}}(d,\delta,m)$ consists of $m$ samples $x_1,\dots, x_m\sim \wh \cD$.
        \item\label{lem:cert_sample-2} The verifier's runtime and protocol's communication complexity are $(m \log\frac{1}{\delta})\cdot\poly d$.
    \end{enumerate}
\end{lemma}

In order to prove \Cref{lem:cert_sample} we leverage our second important tool, certifiable sum, which gives a protocol for the verifier to efficiently determine the answer to arbitrary functions of the form $\sum_{x\in\zo^d} t(x)$ given only query access to $t$.

\begin{lemma}[Certifiable sum]
    \label{lem:cert_sum}
    Fix $\lambda\in\N$.
    \Cref{prot:cert_sum} has the property that
    for all dimensions $d\in\N$ and functions $t:\zo^d\to \lamprec$ the following holds: 
    \begin{enumerate}
        \item 
        For all $b\in\zo$ and $\cP^*_{1-b}$ we have 
        \(
        \brackets{\cP^t_b, \cP^*_{1-b}, \cV^t}(d) = \sum_{x\in\zo^d} t(x).
        \)
        \item The runtime of the verifier and the communication complexity of the protocol are both $\lambda\cdot \poly d$. 
        \item The verifier makes $2$ queries to $t$.
    \end{enumerate}
\end{lemma}

We defer the proof of \Cref{lem:cert_sum} and complete the proof of \Cref{lem:cert_sample} below.

\begin{proof}[Proof of \Cref{lem:cert_sample}]
    Let $\prec$ denote $\prec_d$. At a high level, this algorithm uses the certifiable sum protocol to perform inverse CDF sampling from $\cD$. Roughly, the verifier selects a uniformly random value in $[0,1]$ and then asks for the value of $x$ such that $\sum_{x' \prec x} \cD(x) < p \leq \sum_{x' \preceq x}\cD(x)$. (In reality, to avoid issues related to numerical precision the verifier actually rounds each probability of the PMF to a multiple of $2^{-\lambda}$ and works with a CDF based on the sum of these rounded values.) The verifier uses calls to certifiable sum protocol (\Cref{lem:cert_sum}) to do this entire process efficiently.
    
    Since exactly representing probabilities may require arbitrary precision, the provers cannot hope to send the exact probabilities. In order to circumvent this issue and bound the communication cost of certifiable sample, we will instead obtain $m$ samples from 
    the distribution $\cD_{\lambda}$ defined by $\cD_{\lambda}(x) = \frac{\floor{\cD(x)}_{\lambda}}{\sum_{x\in\zo^d} \floor{\cD(x)}_{\lambda}}$ where $\lambda>d$ is an integer and $\floor{y}_\lambda$ denotes $2^{-\lambda}\cdot\floor{2^{\lambda}\cdot y}$ for all $y\in\R$---that is, $\floor{y}_\lambda$ denotes the nearest multiple of $2^{-\lambda}$ that is at most $y$.

    Define the protocol as follows:

\protocoltop{
$\brackets{\cP_0^{Q_\cD,\prec},\cP_{1}^{Q_\cD,\prec},\cV^{Q_\cD,\prec}}(d, \delta, m)$
}{certifiable sample}{prot:NEW-cert_sample}{
    \begin{enumerate}
        \item $\cV$: Set $\lambda\gets d + 1 + \log\frac{1}{\delta}$. Obtain $T_\lambda = \sum_{x\in\zo^d} \floor{\cD(x)}_\lambda$ via certifiable sum (\Cref{lem:cert_sum}) with $t(x) = \floor{\cD(x)}_\lambda$.
        \item $\cP_0,\cP_1$: Let $S\gets \supp(\cD)$.
        \item Repeat $m$ times:
        \begin{enumerate}
            \item $\cV$: Select $p\sim [0,T_\lambda] \cap 2^{-\lambda}\cdot \Z$ uniformly at random and send to $\cP_0,\cP_1$. \item For each $b\in\zo$:
            \begin{enumerate}
            \item $\cP_b$: return the largest\footnote{according to $\prec$} $\wh x_b$ s.t.\
            \(
                \sum_{\substack{x\in S \text{ s.t.}\\ x\prec \wh x_b }} \floor{\cD(x)}_\lambda  <  p.
            \)
            \item $\cV$: Obtain $p^\mathit{small}_b = \sum_{\substack{x\in S \text{ s.t.}\\ x\prec \wh x_b }} \floor{\cD(x)}_\lambda $ via certifiable sum (\Cref{lem:cert_sum}), using
            \(
                t(x) = \floor{\cD(x)}_\lambda\cdot \Ind\brackets{x\in S \wedge x\prec \wh x_b}.
            \)
            \item $\cV$: If $p^\mathit{small}_b <  p \leq p^\mathit{small}_b + \floor{\cD(\wh x_b)}_\lambda$ then accept $\wh x_b$; otherwise reject. \label{line:rej-cert-samp}
            \end{enumerate}
            \item $\cV$: Let $\wh x$ denote the first value in $(\wh x_0, \wh x_1)$ that was accepted. Output $\wh x$.
        \end{enumerate}
    \end{enumerate}}

    To prove the lemma, it suffices to show (1) that each loop of the protocol obtains an i.i.d.\ sample $\wh x \sim \cD_\lambda$ and (2) that $\dtv(\cD,\cD_\lambda)\leq \delta$.

    We first show item (1). Note that, for fixed $p$, there is a unique value $\wh x_b$, and this a unique value $p^{\mathit{small}}_b$, satisfying the conditions; let $x^*$ and $p^*$ denote these values. 
    Assume without loss of generality that $\cP_0$ is honest. The value $\wh x_0$ is accepted: by the correctness guarantees of certifiable sum (\Cref{lem:cert_sum}), since $\cP_0$ is honest then $\cV$ correctly obtains $\wh x = x^*$ and $p^\mathit{small}_0 = p^*$, so $\wh x_0$ will satisfy the condition checked by $\cV$ on Line~\ref{line:rej-cert-samp}.
    Next note that the value $\wh x_1$ is rejected if $\wh x_1 \neq x^*$, since in this case it fails to satisfy the condition
    \(
    p^\mathit{small}_1 < p \leq p^\mathit{small}_1 + \floor{\cD(\wh x_1)}_\lambda.
    \)
    Thus, by the correctness guarantees of certifiable sum (\Cref{lem:cert_sum}), and since $\cP_0$ is honest, the verifier obtains a value $\wh x = x^*$ satisfying the condition checked on Line~\ref{line:rej-cert-samp}.

    We now prove that if $\wh x$ satisfies the condition $p_\mathit{small} < p \leq p_\mathit{small} + \floor{\cD(\wh x)}_\lambda$, then $\wh x\sim \cD_\lambda$. This follows immediately from
    \begin{align*}
    \Pr[\wh x = x]
    &= \Pr_{p}\brackets{
        \sum_{x' \prec x} \floor{\cD(x')}_\lambda
        < p
        \le
        \sum_{x' \preceq x} \floor{\cD(x')}_\lambda
    } \\
    &= \frac{
        \sum_{x' \preceq x} \floor{\cD(x')}_\lambda
        -
        \sum_{x' \prec x} \floor{\cD(x')}_\lambda
    }{
        \sum_{x\in\zo^d} \floor{\cD(x)}_\lambda
    } \\
    &= \frac{\floor{\cD(x)}_\lambda}{\sum_{x\in\zo^d} \floor{\cD(x)}_\lambda}
    = \cD_\lambda(x).
    \end{align*}
    Because each of the $m$ iterations of the protocol uses an i.i.d.\ value $p\sim\cD_\lambda$, we obtain an i.i.d.\ value $\wh x\sim \cD_\lambda$ from each iteration of the protocol.

    To complete the proof of \Cref{lem:cert_sample-1}, we argue that $\dtv(\cD_\lambda, \cD)\leq \delta$. 

    \begin{claim}
        \label{claim:lambda-precise-distribution}
        Fix $d,\lambda\in\N$ with $\lambda>d$ and a distribution $\cD$ over $\zo^d$. Let $\cD_\lambda$ be the distribution defined by $\cD_{\lambda}(x) = \frac{\floor{\cD(x)}_{\lambda}}{\sum_{x\in\zo^d} \floor{\cD(x)}_{\lambda}}$, where $\floor{y}_\lambda$ denotes $2^{-\lambda}\cdot\floor{2^{\lambda}\cdot y}$ for all $y\in\R$. Then $\dtv(\cD_\lambda,\cD)\leq 2^{d+1-\lambda}$. 
    \end{claim}
    \begin{proof}
        Observe that for each $y\in[0,1]$ we have $\abs{y-\floor{y}_\lambda}\leq 2^{-\lambda}$, and therefore $\abs{\floor{\cD(x)}_\lambda - \cD(x)}\leq 2^{-\lambda}$ for each $x\in\zo^d$.
        It follows that $\sum_{x\in\zo^d}\abs{\floor{\cD(x)}_\lambda - \cD(x)}\leq 2^{d-\lambda}$ and hence, since $\sum_{x\in\zo^d} \cD(x)=1$, that $\sum_{x\in\zo^d}\floor{\cD(x)}_\lambda\in\brackets{1\pm 2^{d-\lambda}}$. 
        Thus, for $T_\lambda=\sum_{x\in\zo^d} \floor{\cD(x)}_\lambda$ we have $T_\lambda\in\brackets{1\pm 2^{d-\lambda}}$.
        Let $a\in [\pm 2^{d-\lambda}]$ be such that $T_\lambda = 1+a$. Then
        \begin{align*}
            \dtv(\cD_\lambda,\cD) 
            & = \frac12\sum_{x\in\zo^d}\abs{\cD_\lambda(x) - \cD(x)}\\
            & = \frac{1}{2\cdot T_\lambda}\sum \abs{\floor{\cD(x)}_\lambda - T_\lambda\cdot\cD(x)}\\
            &\leq \frac{1}{2\cdot T_\lambda}\paren{\sum \abs{\floor{\cD(x)}_\lambda - (1+a)\cdot \cD(x)}}\\
            &\leq \frac{1}{2\cdot T_\lambda}\paren{\sum \abs{\floor{\cD(x)}_\lambda - \cD(x)} + 2^{d-\lambda}\sum \cD(x)}\\
            &\leq \frac{1}{2\cdot T_\lambda}\paren{2^{d-\lambda} + 2^{d-\lambda}}\\
            &\leq \frac{2^{d-\lambda}}{1-2^{d-\lambda}}\leq 2^{d+1-\lambda},
        \end{align*}
        completing the proof.
    \end{proof}

    To see why the communication complexity holds, observe the protocol consists of four calls to certifiable sum and the communication of a constant number of additional terms with $\lambda = d + 1 + \log \frac{1}{\delta}$ bits of precision. Thus, up to constant factors the runtime and communication complexity are inherited from certifiable sum.
\end{proof}

In the remainder of the section we complete the proof of our second tool, the certifiable sum protocol (\Cref{lem:cert_sum}).
At a high level, the protocol works as follows: For all $b\in\zo$ let $C_b=\sets{x\in\zo^d \mid x_1=b}$. First, $\cP_0$ claims that the sum over $\zo^d$ is $\wh T$, and that the sum over $C_0$ and $C_1$ is $\wh T_0$ and $\wh T_1$ respectively. Since $C_0$ and $C_1$ are disjoint and $C_0\cup C_1=\zo^d$, we must have $\wh T_0+ \wh T_1=\wh T$. The key observation is that if $\wh T \neq T$ (the correct value of the sum), then either $\wh T_0\neq T_0$ or $\wh T_1\neq T_1$ (the correct values of the sum on $C_0$ and $C_1$). The verifier can then ask $\cP_1$ for the bit $b$ such that $\wh T_b\neq T_b$, and repeat the above steps on $C_b$. After $d$ rounds there is only a single point remaining in the subcube, and hence the verifier can check if the claim made by $\cP_0$ is correct by making a single query to $t$. The protocol to certify a claim made by $\cP_1$ is identical but has the roles of $\cP_0$ and $\cP_1$ switched. 

\begin{proof}[Proof of \Cref{lem:cert_sum}]
    Our protocol consists of two phases. In the first phase the verifier uses $\cP_0$ to verify the claim made by $\cP_1$, and in the second phase the verifier uses $\cP_1$ to verify the claim made by $\cP_0$. The final protocol executes both phases and returns the first verified claim (this will be correct if at least one prover is honest). We first define the following notation: for all $t:\zo^d\to \lamprec$ and points $z\in\zo^j$ for some $j<d$, let
     \[
     t_z(x) = 
     \begin{cases} 
          t(x) &  \text{if }x_i=z_i\text{, for all $i\in[j]$;}  \\
          0 & \text{otherwise,} 
    \end{cases}
    \]
    and define $T_z = \sum_{x\in\zo^d} t_z(x)$. Fix $b\in\zo$, and define the protocol $\brackets{\cP_0(t),\cP_1(t),\cV^t}_b$ in \Cref{prot:cert_sum}.
     \protocoltop{$\brackets{\cP^t_0,\cP^t_1,\cV^t}_b(d)$}{certifiable sum}{prot:cert_sum}{
    \begin{enumerate}
        \item $\cP_b$: Let $z\gets\emptyset$. Send $(\wh T_z, \wh T_{z0}, \wh T_{z1}) \gets (T_z, T_{z0}, T_{z1})$ to $\cV$.
        \item $\cV$: %
        If $\wh T_z\neq \wh T_{z0} + \wh T_{z1}$ then output $\perp$. Otherwise, send $(\wh T_z, \wh T_{z0}, \wh T_{z1})$ to $\cP_{1-b}$.
        \item $\cP_{1-b}$: If there exists $j\in\zo$ such that $T_{zj}\neq \wh T_{zj}$ then send $j$ to $\cV$. Otherwise send $0$. 
        \item $\cV$: If $|z|<d-1$ then send $j$ to $\cP_b$ and repeat the protocol with $\wh T_z\gets \wh T_{zj}$ and $z\gets zj$. Otherwise, accept if and only if $t({zj})=\wh T_{zj}$. \label{line:cert-sum-acc-rej} %
    \end{enumerate}
    }

    Assume without loss of generality that $\cP_0$ is honest. First, we prove that for all $\cP^*_1$ the protocol $\brackets{\cP^t_0,\cP^*_1,\cV^t}_0$ accepts. %
    The proof proceeds by induction on $d$. For the base case, suppose $d=1$. Then, since $\cP_0$ is honest, $\wh T = T$, $\wh T_0 = T_0$, and $\wh T_1 = T_1$. Thus, for all $j\in\zo^d$ we have $t(j) = \wh T_j$, and thus the verifier always accepts. Now, suppose the claim holds up to some $d\in\N$ and consider the case of $d+1$. By the same argument as the base case, after the first round of the protocol we have $\wh T_j = T_j$. Since $t_j(x) = 0$ whenever $x_1\neq j$, round two of the protocol is identical to executing the protocol with function $t'_j:\zo^{d}\to \Q_{\lambda}$ given by $t'_j(x) = t_j(jx)$. Since the domain of $t'_j$ is $d$, the inductive hypothesis implies that the verifier will accept. 

    Next, we prove that $\brackets{\cP^t_0,\cP^*_1,\cV^t}_1$ rejects for all $\cP^*_1$ that send $\wh T\neq T$ in the first round. As before, the proof follows by induction on $d$. Consider the base case of $d=1$. If $\wh T\neq T$ and $\wh T_0 + \wh T_1 = \wh T$, then either $\wh T_0\neq T_0$ or $\wh T_1\neq T_1$. Since $\cP_0$ is assumed to be honest, it will send $j\in\zo$ such that $\wh T_j\neq T_j = t(j)$. Since $t(j)\neq \wh T_j$ the verifier will reject. Now, suppose the claim holds up to some $d\in\N$, and consider the case of $d+1$. If $\wh T\neq T$ in the first round, then $\wh T_j\neq T_j$ for some $j\in \zo$. Since $\cP_0$ is assumed to be honest, it will send $\cV$ the $j\in\zo$ such that $\wh T_j\neq T_j$. Observe that the next round of the protocol is identical to executing the protocol with the function $t'_j:\zo^d\to \Q_{\lambda}$ given by $t'_j(x) = t_j(jx)$, and with $\wh T'\neq T' = \sum_{x\in\zo^d} t'_j(x)$. By the inductive hypothesis, the verifier rejects.     

    To complete the proof, define the protocol $\brackets{\cP^t_0,\cP^t_1,\cV^t}$ as follows: for each $b\in\zo$ run protocols $\brackets{\cP^t_0,\cP^t_1,\cV^t}_b$, and return the $\wh T$ from the first round of an accepting execution. By the above arguments, the protocol $\brackets{\cP^t_0,\cP^t_1,\cV^t}$ always outputs $T$. 

    To see why the communication complexity holds, observe that sum $\sum_{x\in\zo^d} t(x)$ is bounded above by $2^{d+2\lambda}$ and is an integer multiple of $2^{-\lambda}$. Thus, $\wh T_z$ will require at most $O(d\lambda)$ bits to send. Since the protocol uses $2d$ rounds, and in each round $O(d\lambda)$ bits are required to transmit $(\wh T_z, \wh T_{z0}, \wh T_{z1})$, it immediately follows that the overall communication complexity is $\lambda\poly d$. The runtime and query complexity of the verifier follows by inspection of \Cref{prot:cert_sum}.    
\end{proof}

\subsection{Refereed query delegation}
\label{sec:ref-query-del}

In this section, we show that a protocol where all parties are given query access to a deterministic oracle $\acc$ can be modified to a protocol where the verifier offloads nearly all of its queries to the provers and only makes a single query to $\acc$. The modification essentially preserves the guarantees of the original protocol (up to a small cost in communication complexity).  %
At a high level, the modification works as follows: each time the verifier would make a query to the oracle, it instead has each prover make that query to the oracle. Each prover then sends the query answer to the verifier. If the query answers match, the verifier continues the protocol with this answer; if the query answers do not match, then the verifier issues a single query to the oracle to figure out the true query answer, and then continues the protocol using only the query answers from the correct prover going forward.

\begin{lemma}[Refereed query delegation]
    \label{lem:cert_query}
    Fix a deterministic oracle $\acc$.
    Let $\brackets{\cP_0,\cP_1,\cV}$ be a protocol that, for all inputs $\kappa\in\R$ and $\lambda\in\N$, has communication complexity $C(\kappa,\lambda)$, verifier runtime $T_\cV(\kappa,\lambda)$, verifier query complexity $q_\cV(\kappa,\lambda)$, prover runtime $T_\cP(\kappa,\lambda)$, and prover query complexity $q_\cP(\kappa,\lambda)$. Additionally, assume all queries to $\acc$ and their answers can be specified using at most $\lambda$ bits.
    
    Then for all $b\in\zo$ and $\wt\cP^*_{1-b}$ there exists a $\cP^*_{1-b}$ such that
    \[
        \brackets{\wt\cP^\acc_b, \wt\cP^*_{1-b},\wt\cV^\acc}(\kappa,\lambda)
        =
        \brackets{\cP^\acc_b,\cP^*_{1-b},\cV^\acc}(\kappa,\lambda),
    \]
    where $\brackets{\wt\cP_0,\wt\cP_1,\wt\cV}$ are described in  \Cref{prot:query_offload}. Moreover, $\brackets{\wt\cP_0,\wt\cP_1,\wt\cV}$ has communication complexity $C(\kappa,\lambda) + 2\lambda\cdot q_\cV(\kappa,\lambda)$, verifier runtime $T_\cV(\kappa,\lambda)$, prover runtime $T_\cP(\kappa,\lambda)+ q_\cV(\kappa,\lambda)$, prover query complexity $q_\cP(\kappa,\lambda) + q_\cV(\kappa,\lambda)$, and verifier query complexity at most $1$.  
\end{lemma}

\begin{proof}[Proof of \Cref{lem:cert_query}]
    Consider the following protocol:

    \protocoltop{$\brackets{\wt\cP^{\acc}_0,\wt\cP^{\acc}_1,\wt\cV^\acc}(\kappa,\lambda)$}{refereed query delegation}{prot:query_offload}{
    \begin{enumerate}
        \item Simulate $[\cP_0^\acc,\cP_1^\acc,\cV^\acc](\kappa,\lambda)$, except answer $\cV$'s queries to $\acc$ using the following procedure:
        \begin{enumerate}
            \item $\wt\cV$ asks both $\wt\cP_0$ and $\wt\cP_1$ to answer the query using their access to $\acc$.
            \item If both provers agree, continue the protocol using this query answer.
            \item If the provers disagree, $\wt\cV$ makes a single query to $\acc$ to determine the lying prover $(1- b)$, and then uses the answers from 
            $\wt\cP_{b}$ for all subsequent queries to $\acc$.\footnote{For the competing prover setting: At the cost of one additional query, the verifier can catch a malicious $\wt\cP_b$ in this phase by asking $\wt\cP_{1-b}$ which of $\wt\cP_b$'s queries are incorrect at the end of the protocol.}
        \end{enumerate}
    \end{enumerate}
    }
   The communication complexity, runtime, and query complexity follow immediately from the fact that \Cref{prot:query_offload} simply runs $[\cP_0,\cP_1,\cV]$ once, and has each prover make at most $q_\cV(\kappa,\lambda)$ additional queries to $\acc$.

    Consider the protocol above run with (malicious) prover $\wt\cP^*_{1-b}$.
    We show there exists a prover $\cP^*_{1-b}$ such that
    \[
        \brackets{\wt\cP^\acc_b, \wt\cP^*_{1-b},\wt\cV^\acc}(\kappa,\lambda)
        =
        \brackets{\cP^\acc_b,\cP^*_{1-b},\cV^\acc}(\kappa,\lambda).
    \]

    Let $\cP^*_{1-b}$ be the prover that is identical to $\wt\cP^*_{1-b}$ (except it simulates query requests from $\cV$ and does not actually send query answers to $\cV$).\footnote{Morally, think of $\cP^*_{1-b}$ and $\wt\cP^*_{1-b}$ as identical. However, a subtlety arises that the behavior of $\wt\cP^*_{1-b}$ may depend on the queries it answers for $\cV$---e.g., $\wt\cP^*_{1-b}$ may execute strategy ``a'' when asked an oracle query by $\cV$, and strategy ``b'' when $\cV$ does not ask any oracle queries. To ensure that $\cP^*_{1-b}$ executes the correct malicious prover strategy, it simulates $\wt\cP^*_{1-b}$ with the queries from $\cV$.}
    We see that, if $\cV$ is simulated using queries to $\acc$ that are answered correctly, then \Cref{prot:query_offload} has exactly the same distribution of outputs as $\brackets{\cP^\acc_b,\cP^*_{1-b},\cV^\acc}(\kappa,\lambda)$.

    To complete the proof, we show that verifier $\wt\cV$ correctly answers all of $\cV$'s queries to $\acc$ using at most $1$ query to $\acc$. First recall that the honest prover $\wt\cP_b$ always answers queries truthfully. Additionally, $\wt\cV$ determines which prover is telling the truth at the first instance on which $\wt\cP^*_{1-b}$ and $\wt\cP^\acc_b$ disagree by making a single query to $\acc$. After this query $\wt\cV$ only uses answers provided by the honest prover, and hence all of the answers it provides to $\cV$ are correct.
\end{proof}

\section{Refereed learning protocols}
\label{sec:rlp}

In this section we prove our two main results. In \Cref{sec:rlp-zo} we construct a $(1+\eps)$-refereed learning protocol for the zero-one loss function; in \Cref{sec:rlp-gen} we construct a $(3+\eps)$-refereed learning protocol for metric loss functions.

Throughout this section we use the following simpler version of \Cref{def:rlp_concrete}, which only has a multiplicative slack term and fixes the oracle access model. Recall that $Q_\cD$ denotes the probability mass function of a distribution $\cD$.

\begin{definition}[Refereed learning protocol---multiplicative error with fixed oracles]
    \label{def:rlp_concrete_mult}
    Fix a range $\cY$ with metric $\ell$, dimension $d\in\N$, slack $\alpha\geq 1$, and soundness error $\beta\geq 0$. Let $\cH\subset\sets{h:\zo^d\to\cY}$ and $\mathds D\subset \sets{\cD\mid \supp(\cD)\subset\zo^d}$. A protocol $\brackets{\cP_0,\cP_1,\cV}$ is an \emph{$(\alpha,\beta)$-refereed learning protocol for $\cH$ and $\mathds D$ with respect to $\ell$}, if for all distributions $\cD\in\mathds D$, functions $f:\zo^d\to\cY$ and $h_0,h_1\in\cH$, the following holds:
     \begin{itemize}
        \item For all $b\in\zo$ and $\cP^*_{1-b}$, the bit $\vb\gets \brackets{\cP^{h_0,h_1,Q_\cD}_b, \cP^*_{1-b}, \cV^{f,h_0,h_1,Q_\cD}}$ satisfies        
    \[
        \Pr\brackets{\cL_{\cD}(h_\vb,f \mid \ell) \leq \alpha\cdot \min_{s\in\zo} \cL_{\cD}(h_s, f \mid \ell)} \geq 1-\beta,
    \]
    where the randomness is over the coins of the verifier, the honest prover, and the oracles. 
    \end{itemize}
\end{definition}

The provers in \Cref{def:rlp_concrete_mult} do not have query access to $f$, and in the protocols of \Cref{thm:rlp_zo,thm:rlp_gen} the verifier makes a number of queries to $f$ that depends only on $\alpha$.
In \Cref{sec:offload-queries} we
describe how, if we also give the provers query access to $f$,
the verifier can offload all but a single query to the provers---that is, we show how the protocols can be modified so that the verifier makes at most one query to either $h_0,h_1,f$, or $Q_\cD$.

\subsection{Refereed learning for zero-one loss}
\label{sec:rlp-zo}

We first consider refereed learning protocols for the special case of 
the \emph{zero-one metric} defined by $\lzo(y,y')=1$ if $y\neq y'$ and $0$ otherwise. \Cref{thm:rlp_zo} states that for all $\eps,\beta>0$ there exists an $(\alpha,\beta)$-refereed learning protocol for $\alpha=1+\eps$, with respect to $\lzo$.

\begin{theorem}[Refereed learning protocol for zero-one loss]
    \label{thm:rlp_zo}
    Fix range $\cY$ and $\lambda\in\N$. There exists a protocol $\brackets{\cP_0,\cP_1,\cV}$ that, for all inputs $d\in\N$ and $\eps,\beta>0$, is a $(1+\eps,\beta)$-refereed learning protocol for $\fF$ and $\mathds D_{\lambda}$ with respect to $\lzo$. The protocol has communication complexity and verifier runtime $\lambda\paren{1+\frac{1}{\eps}}^2\log\paren{1+\frac{1}{\eps}}\log\frac1\beta\cdot\poly d = \wt O_{\lambda,\beta}\bparen{\paren{1+\frac{1}{\eps}}^2\cdot \poly d}$, and the verifier makes $O\bparen{\paren{1+\frac{1}{\eps^2}}\log\frac1\beta}$ queries to $f$.
\end{theorem}

\begin{proof}[Proof of \Cref{thm:rlp_zo}]
    We define the protocol $\brackets{\cP_0,\cP_1,\cV}$ in \Cref{prot:zo_rlp}.

    \protocoltop{$\brackets{\cP^{h_0,h_1,Q_\cD}_0,\cP^{h_0,h_1,Q_\cD}_1,\cV^{f,h_0,h_1,Q_\cD}}\paren{d,\eps,\beta}$}{refereed learning for zero-one loss}{prot:zo_rlp}{
    \begin{enumerate}
        \item $\cP_0,\cP_1$: Let $S\gets\sets{x\in\zo^d \mid h_0(x)\neq h_1(x)}$. %
        \item $\cV$: Obtain $p_S=\cD(S)$ using certifiable sum (\Cref{lem:cert_sum}) with $t(x)=\Ind\brackets{x\in S}\cdot\cD(x)$ and $\lambda\gets \lambda$. If $p_S=0$ then output $\vb\sim\zo$.
        \item $\cV:$ Set $\delta \gets \frac{\eps}{4(2+\eps)}$ and $m\gets \frac{\log1/\beta}{\delta^2}$. Execute certifiable sample (\Cref{lem:cert_sample}) with distribution $\cD|_S$ to draw $m$ samples $x_1,\dots, x_m\sim\wh\cD$ with distance parameter $\delta$.
        \item $\cV:$ Query $f$ on $x_1,\dots, x_m$ and output $\vb=\arg\min_{s\in\zo}\abs{\sets{i\in[m] : h_s(x_i)\neq f(x_i)}}$.
    \end{enumerate}}
    
    First, we argue that \Cref{prot:zo_rlp} satisfies the soundness condition of \Cref{def:rlp_concrete_mult}. Assume without loss of generality that $\cL_\cD(h_1,f\mid \lzo) > (1+\eps)\cdot \cL_\cD(h_0,f\mid \lzo)$. Since $\lzo$ is the zero-one metric, this is equivalent to the assumption that $\Pr_\cD\brackets{h_1(x)\neq f(x)}>(1+\eps)\cdot \Pr_\cD\brackets{h_0(x)\neq f(x)}$. To prove that the verifier in \Cref{prot:zo_rlp} outputs the correct bit, we prove \Cref{claim:distance_test}, which roughly states that for $x\sim \cD|_S$, we have $h_1(x)\neq f(x)$ with probability at least $\frac 12 + \frac{\eps}{1+\eps}$. 
    
    \begin{claim}
    \label{claim:distance_test}
    Fix $d\in\N$, functions $f,h_0,h_1:\zo^d\to\cY$, distribution $\cD$ over $\zo^d$, and $\eps>0$. Assume $\Pr_\cD\brackets{h_1(x)\neq f(x)} > (1+\eps)\cdot \Pr_\cD\brackets{h_0(x)\neq f(x)}$, and let $S= \sets{x : h_0(x)\neq h_1(x)}$. Then,
    \[
    \Pr_{x\sim \cD|_S}\brackets{h_0(x)\neq f(x)}< \frac 12 - \frac{\eps}{2(2+\eps)}.
    \]
    \end{claim}

\begin{proof}
    By the hypothesis on $h_0,h_1$, and $f$, and the law of total probability,
    \begin{align}
    &\eps\cdot\Pr_{x\sim\cD}\brackets{h_0(x)\neq f(x)}< \Pr_{x\sim \cD}\brackets{h_1(x)\neq f(x)} - \Pr_{x\sim \cD}\brackets{h_0(x)\neq f(x)}\label{eq:distance_test-1}\\
    &=\paren{\Pr_{x\sim \cD}\brackets{h_1(x)\neq f(x) \mid x\in S} - \Pr_{x\sim \cD}\brackets{h_0(x)\neq f(x) \mid x\in S}}\cdot\Pr_{x\sim\cD}\brackets{x\in S}\nonumber\\
    &+\paren{\Pr_{x\sim \cD}\brackets{h_1(x)\neq f(x) \mid x\not\in S} - \Pr_{x\sim \cD}\brackets{h_0(x)\neq f(x) \mid x\not\in S}}\cdot\Pr_{x\sim\cD}\brackets{x\not\in S}.\nonumber
    \end{align}
    Since $h_0(x)=h_1(x)$ for all $x\not\in S$, the difference in the second term of the sum is
    \[
    \Pr_{x\sim \cD}\brackets{h_1(x)\neq f(x) \mid x\not\in S} - \Pr_{x\sim \cD}\brackets{h_0(x)\neq f(x) \mid x\not\in S}=0.
    \]
    so rearranging \eqref{eq:distance_test-1} yields 
    \begin{align}
    \label{eq:distance_test-2}
    \Pr_{x\sim \cD}\brackets{h_1(x)\neq f(x) \mid x\in S} - \Pr_{x\sim \cD}\brackets{h_0(x)\neq f(x) \mid x\in S}> \frac{\eps\cdot\Pr_{x\sim\cD}\brackets{h_0(x)\neq f(x)}}{\Pr_{x\sim\cD}\brackets{x\in S}}.
    \end{align}
    By the definition of $S$, the events $h_0(x)\neq f(x)$ and $h_1(x)\neq f(x)$ are disjoint, and hence
    \begin{align*}
        1 
        &= \Pr_{\cD|_S}\brackets{h_1(x)\neq f(x)}+\Pr_{\cD|_S}\brackets{h_0(x)\neq f(x)}\\
        &= \Pr_{\cD|_S}\brackets{h_1(x)\neq f(x)}-\Pr_{\cD|_S}\brackets{h_0(x)\neq f(x)}+2\Pr_{\cD|_S}\brackets{h_0(x)\neq f(x)}\\
        &> \frac{\eps\cdot\Pr_{x\sim\cD}\brackets{h_0(x)\neq f(x)}}{\Pr_{x\sim\cD}\brackets{x\in S}} + 2\Pr_{\cD|_S}\brackets{h_0(x)\neq f(x)}\\
        & \geq (2+\eps)\Pr_{x\sim\cD|_S}\brackets{h_0(x)\neq f(x)}
    \end{align*}
    where the second to last inequality follows from \eqref{eq:distance_test-2}, and the last equality follows by applying the law of total probability to the numerator. Rearranging terms yields the desired conclusion.%
\end{proof}

    To complete the proof of \Cref{thm:rlp_zo}, observe that by \Cref{lem:cert_sum}, the verifier correctly obtains $p_S = \cD(S)$. Since $\cD|_S(x) = \frac{\cD(x)\cdot\Ind\brackets{x\in S}}{\cD(S)}$, the verifier has query access to $Q_{\cD|_S}$, the probability mass function of $\cD|_S$. Thus, by \Cref{lem:cert_sample}, we have $\dtv\bparen{\wh\cD,\cD|_S}\leq \delta$ and therefore, by \Cref{claim:distance_test} and the definition of $S$, we have
    \[
        \Pr_{x\sim \wh\cD}\brackets{h_0(x)\neq f(x)}
        <
        \frac 12 - \frac{\eps}{2(2+\eps)}+\delta.
    \]
    To see why the verifier outputs $0$ with probability at least $1-\beta$, let $\wh p = \frac1m\sum_{i\in [m]}\Ind\brackets{h_0(x_i)\neq f(x_i)}$.
    Since $\Ex[\wh p]=\Pr[h_0(x)\neq f(x)]$, Hoeffding's inequality and our setting of $m$ in \Cref{prot:zo_rlp} implies  
    \[
    \Pr\brackets{\abs{\wh p - \Pr[h_0(x)\neq f(x)]}\geq \delta}\leq 2\exp\paren{-2m/\delta^2}<\beta.
    \]
    It follows that $\wh p < \frac 12 - \frac{\eps}{2(2+\eps)} + 2\delta \leq \frac 12$ with probability at least $1-\beta$, and therefore $\cV$ will output $\vb=0$ with probability at least $1-\beta$. The runtime and communication complexity follow by inspection of \Cref{prot:zo_rlp}, and from the guarantees of \Cref{lem:cert_sample,lem:cert_sum}.    
\end{proof}

\subsection{Refereed learning for metric loss functions}
\label{sec:rlp-gen}

We next consider refereed learning protocols for any metric loss function (\Cref{def:loss_function}). \Cref{thm:rlp_gen} states that for all $\eps,\beta > 0$ there exists an $(\alpha, \beta)$-refereed learning protocol for $\alpha = 3 + \eps$, with respect to the chosen metric loss function. While the slack parameter $\alpha = 3 + \eps$ is worse (as compared to $\alpha = 1 + \eps$), this general protocol has the same time, communication, and query complexity guarantees as the protocol in \Cref{thm:rlp_zo}.

\begin{theorem}[Refereed learning protocol for general loss functions]
    \label{thm:rlp_gen} 
    Fix range $\cY$, $\lambda\in\N$, and $\lambda$-precise metric $\ell$ on $\cY$. There exists a protocol $\brackets{\cP_0,\cP_1,\cV}$ that, for all inputs $d\in\N$ and $\eps,\beta>0$, is a $(3+\eps,\beta)$-refereed learning protocol for $\fF$ and $\mathds D_\lambda$ with respect to $\ell$. The protocol has communication complexity and verifier runtime $\lambda\paren{1+\frac{1}{\eps^2}}\log\paren{1+\frac{1}{\eps}}\log\frac1\beta\cdot\poly d=\wt O_{\lambda,\beta}\bparen{\paren{1+\frac{1}{\eps^2}}\cdot\poly d}$, and the verifier makes $O\bparen{\paren{1+\frac{1}{\eps^2}}\log\frac1\beta}$ queries to $f$.
\end{theorem}

\begin{proof}[Proof of \Cref{thm:rlp_gen}]
    In order to construct our protocol, we first introduce a scaled version of the distribution $\cD$ called $\drescale{h_0, h_1}$. Intuitively, $\drescale{h_0,h_1}$ assigns more probability mass to points $x$ where $\ell\bigl(h_0(x), h_1(x)\bigr)$ is large.
    Since whenever $\ell\bigl(h_0(x), h_1(x)\bigr)$ is large, it must be the case that either $\ell\bigl(h_0(x), f(x)\bigr)$ or $\ell\bigl(h_1(x), f(x)\bigr)$ is large, sampling points $x$ with higher probability where $\ell\bigl( h_0(x), h_1(x) \bigr)$ is larger allows the verifier to distinguish $h_0$ and $h_1$ more easily. 

    \begin{definition}[Loss-rescaled distribution]
    \label{def:drescale}
        For all $d\in\N$, distributions $\cD$ over $\zo^d$, sets $\cY$ with metric $\ell$, and functions $h_0,h_1:\zo^d \to \cY$, define the \emph{loss-rescaled distribution} $\drescale{h_0, h_1}$ via the density
        \[
            \drescale{h_0, h_1}(x)
            :=
            \cD(x) \cdot
            \frac{\ell\bigl(h_0(x), h_1(x)\bigr)}{\Ex_{x\sim \cD}\brackets{\ell\bigl(h_0(x), h_1(x)\bigr)}}.
        \]
    \end{definition}

    We define the protocol $\brackets{\cP_0,\cP_1,\cV}$ in \Cref{prot:gen_rlp}.
    
     \protocoltop{$\brackets{\cP^{h_0,h_1,Q_\cD}_0,\cP^{h_0,h_1,Q_\cD}_1,\cV^{f,h_0,h_1,Q_\cD}}(d,\eps,\beta)$}{refereed learning for metric loss}{prot:gen_rlp}{
    \begin{enumerate}
        \item $\cV$: Execute certifiable sum (\Cref{lem:cert_sum}) with $t(x) := \ell(h_0(x),h_1(x))\cdot \cD(x)$ and $\lambda\gets2\lambda$ to compute $\mu\gets \Ex_{x\sim \cD}\brackets{\ell\bigl(h_0(x),h_1(x)\bigr)}$. If $\mu=0$ then output $\vb\sim\zo$. 
        \item $\cV$: Set $\delta\gets \frac{\eps}{4(2+\eps)}$ and $m\gets \frac {\log1/\beta}{\delta^2}$. Execute certifiable sample (\Cref{lem:cert_sample}) with distribution $\drescale{h_0, h_1}$ to draw $m$ samples $x_1,\dots, x_m\sim\wh\cD$ with distance parameter $\delta$.
        \item $\cV$: Query $f$ on $x_1,\dots, x_m$ and for each $b\in\zo$ let 
        \[
        \wh R_b \gets \frac{1}{m}\sum_{i\in[m]} \frac{\ell\bigl(h_b(x_i),f(x_i)\bigr)}{\ell\bigl(h_0(x_i),f(x_i)\bigr) + \ell\bigl(h_1(x_i),f(x_i)\bigr)}.
        \]
        \item $\cV$: Output $\vb \gets \arg\min_{b\in\zo} \wh R_b$        %
    \end{enumerate}}

    In order to prove the soundness condition, we argue that the bit $\vb$ output by the verifier satisfies $\cL_\cD(h_\vb,f)\leq (3+\eps)\cL_\cD(h_{1-\vb},f)$ with probability at least $1-\beta$. First, recall that by \Cref{lem:cert_sum,lem:cert_sample}, the verifier $\cV$ obtains the expectation $\Ex_{x\sim\cD}\brackets{\ell(h_0(x), h_1(x))}$, and $m$ samples $x_1,\dots, x_m$ from a distribution $\wh \cD$ such that $\dtv\bparen{\wh\cD, \drescale{h_0, h_1}}\leq \delta$. We assume without loss of generality that $\Ex_{x\sim \cD}\brackets{\ell\bigl(h_0(x),h_1(x)\bigr)}\neq 0$ since otherwise $h_0=h_1$.
    
    The proof of correctness proceeds in two major steps. First, in \Cref{claim:gen-dist-test} we will show that for all $b\in\zo$ the statistic defined by 
    \begin{align}
    r_b(x) := \frac{\ell\bparen{h_b(x),f(x)}}{\ell\bparen{h_0(x),f(x)} + \ell\bparen{h_1(x),f(x)}}
    \quad \text{and} \quad
    R_b := \Ex_{x\sim \drescale{h_0, h_1}}
        \brackets{r_b(x)},\label{eq:R_def}
    \end{align}
    is less than $\frac 12 - \eps$ whenever $\cL_\cD(h_{1-b},f\mid \ell)>(3+\eps)\cdot \cL_\cD(h_b,f\mid \ell)$. Then, we will argue that the estimate $\wh R_b$ computed by $\cV$ in \Cref{prot:gen_rlp} is concentrated around $\Ex\brackets{\wh R_b}$, and that $R_b$ and $\Ex\brackets{\wh R_b}$ are close together. Combining the above with the fact that $R_0+R_1=1$ suffices to complete the proof. %

    \begin{claim}
    \label{claim:gen-dist-test}
        Fix a set $\cY$ with metric $\ell$, dimension $d\in\N$, distribution $\cD$ over $\zo^d$, functions $f,h_0,h_1:\zo^d\to\cY$, and $\eps>0$.
        For all $b\in\zo$, if $\cL_\cD\paren{h_{1-b}, f\mid \ell} > (3+\eps)\cdot \cL_\cD\paren{h_b, f \mid \ell}$ then
        \[
        R_b < \frac{1}{2} - \frac{\eps}{2(2+\eps)},
        \]
        where $R_b$ is defined in \eqref{eq:R_def}.
    \end{claim}

    \begin{proof}
        To avoid cluttered expressions, let
        $\ell_b(x) := \ell(h_b(x), f(x))$, %
        let $\Delta(x) := \ell(h_0(x), h_1(x))$ and let $\mu := \Ex_{x\sim \cD}[\Delta(x)]$.
        First, by the definition of $R_b$ and $\drescale{h_0, h_1}$, we have
        \begin{align*}
            R_b
            &=
            \Ex_{x\sim \drescale{h_0, h_1}}
            \brackets{
                \frac{\ell_b(x)}
                {\ell_0(x) + \ell_1(x)}
            } 
            =
            \Ex_{x\sim \cD}
            \brackets{
                \frac{\ell_b(x)}
                {\ell_0(x) + \ell_1(x)}
                \cdot
                \frac{\Delta(x)}
                {\mu}
            } 
            \leq
            \frac{\Ex_{x\sim \cD}
            \brackets{\ell_b(x)}}{\mu}, %
        \end{align*}
        where the last inequality follows since $\Delta(x) \leq \ell_0(x) + \ell_1(x)$ by the triangle inequality. Next, we apply the assumption that $\cL_\cD\paren{h_{1-b}, f\mid \ell} > (3+\eps)\cdot \cL_\cD\paren{h_b, f \mid \ell}$ to show that $\mu$ can be lower bounded as
        \begin{align*}
            \mu 
            = \Ex_{x\sim \cD} \brackets{\Delta(x)} 
            \geq \Ex_{x\sim\cD} \brackets{\ell_{1-b}(x) - \ell_b(x)} %
            > (2 + \eps) \cdot \Ex_{x\sim \cD}\brackets{\ell_b(x)}. %
        \end{align*}
        Substituting this bound on $\mu$ into our bound on $R_b$ gives us $R_b < \frac{1}{2+\eps}=\frac12 - \frac{\eps}{2(2+\eps)}$.%
    \end{proof}

    Now, since $r_b(x)\in [0,1]$ and each $x_i$ in \Cref{prot:gen_rlp} is sampled from $\wh\cD$ independently, Hoeffding's inequality implies %
    \begin{align}
        \Pr
        \brackets{ \abs{\wh R_b - \Ex\brackets{\wh R_b}} \geq \delta } \leq 2\exp\paren{-2m\delta^2}< \frac\beta2.\label{eq:r_hat_concentration}
    \end{align}
    Next, we bound the distance between $R_b$ and $\Ex\brackets{\wh R_b}$. Recall $R_b=\Ex_{x\sim \drescale{h_0,h_1}}\brackets{r_b(x)}$, and $\Ex\brackets{\wh R_b} = \Ex_{x\sim\wh \cD}\brackets{r_b(x)}$. Since $r_b(x)\in [0,1]$ and $\dtv\bparen{\wh\cD, \drescale{h_0, h_1}}\leq \delta$, we have
    \begin{align*}
        \abs{\Ex\brackets{\wh R_b}-R_b}
        &=
        \abs{\Ex_{x\sim \wh\cD}\brackets{r_b(x)}-\Ex_{x\sim \drescale{h_0, h_1}}\brackets{r_b(x)}} \\
        &= \abs{ \sum_{x\in\zo^d} \paren{ \wh\cD(x) - \drescale{{h_0,h_1}}(x) } \cdot r_b(x) }\\
        &\leq \abs{ \sum_{x\in\zo^d} \paren{ \wh\cD(x) - \drescale{{h_0,h_1}}(x) } \cdot \paren{r_b(x)-\frac 12}} + \abs{ \sum_{x\in\zo^d} \paren{ \wh\cD(x) - \drescale{{h_0,h_1}}(x) } \cdot \frac12 }\\
        &\leq\dtv(\wh\cD,\drescale{h_0,h_1})\leq \delta.
    \end{align*}

    Combining \eqref{eq:r_hat_concentration} and the above bound %
    yields $\abs{\wh R_b - R_b}\leq 2\delta$ for each $b\in\zo$ with probability at least $1-\beta$. To complete the proof, suppose $\cL_\cD\paren{h_{1-s},f}>(3+\eps)\cL_\cD\paren{h_s,f}$ for some $s\in\zo$. By \Cref{claim:gen-dist-test} we have $R_s<\frac12-\frac{\eps}{2(2+\eps)}$, and since $R_0+R_1=1$, this implies that $R_{1-s}>\frac 12 + \frac{\eps}{2(2+\eps)}$. If $|\wh R_b - R_b|\leq 2\delta$ for each $b\in\zo$ then by our choice of $\delta$ we have $\wh R_s<\frac 12<\wh R_{1-s}$. Since $\rho = \arg\min_{b\in\zo} \wh R_b$, we see that the verifier outputs $\vb=s$ with probability at least $1-\beta$. %
    The runtime, communication complexity, and query complexity guarantees follow from \Cref{lem:cert_sum,lem:cert_sample}, and by inspection of \Cref{prot:gen_rlp}.
\end{proof}

\subsection{Offloading queries to the provers}
\label{sec:offload-queries}

The protocols in the proofs of \Cref{thm:rlp_zo,thm:rlp_gen} do not have the provers make any queries to $f$.
However, if we give the provers query access to $f$, then  we can use the  ``refereed query delegation'' technique of \Cref{lem:cert_query} to offload all verifier queries to the provers, while keeping the complexity of the protocol essentially unchanged.
The resulting protocols incur an additional factor of $d+\lambda+\log|\cY|$ in communication complexity; however, the verifier in this modified protocol only makes at most $1$ query to either $f,h_0,h_1$, or $Q_\cD$.

\section{Lower bounds for refereed learning}
\label{sec:lower_bounds}

In this section we prove several lower bounds for refereed learning protocols with ``white-box'' access to $h_0$ and $h_1$---that is, the protocols receive a representation of $h_0$ and $h_1$ as input. 
Since the white-box versions of $\cP_0,\cP_1$, and $\cV$ can simulate their black-box counterparts, a lower bound against white-box refereed learning implies the same lower bound against the black-box version.\footnote{There is a caveat that lower bounds on runtime depend on the representation and may incur a factor that depends on the time complexity of evaluating $h_0$ and $h_1$. We deal with this issue explicitly in \Cref{sec:prover_runtime_lb}. On the other hand, white-box query and sample complexity lower bounds apply directly to the black-box setting.}
In what follows we show that for simple classes of functions and distributions, even white-box refereed learning protocols require: (1) query access to $f$ (\Cref{sec:lb-queries-needed}), (2) query access to $Q_\cD$ (\Cref{sec:dist_knowledge}), and (3) exponential-time provers (\Cref{sec:prover_runtime_lb}).

\paragraph{Lower bounds for weaker verifier access models.} In \Cref{thm:lower_bound_samples,thm:dist_knowledge}, we consider refereed learning protocols with additive and multiplicative error ($\alpha\geq 1$ and $\eta\in(0,1)$), and show that if instead of query access to $f$ and $Q_\cD$ the verifier either (1) has query access to $Q_\cD$ but only has access to $f$ via random labeled samples $(x,f(x))$, or (2) has query access to $f$, but only has access to $Q_\cD$ via samples $x\sim\cD$, then it requires sample complexity at least $\frac 1\eta$. This immediately implies that when $\eta\to 0$ (the setting of \Cref{thm:rlp_zo,thm:rlp_gen}), every refereed learning protocol requires verifier query access to $f$ and $Q_\cD$.

\paragraph{Time complexity lower bound.} In \Cref{thm:rlp_runtime_lb} we focus on the setting of $\eta=0$, and show how a refereed learning protocol can be used to decide 3-SAT with a constant factor overhead in running time. Subject to standard computational hardness assumptions, \Cref{thm:rlp_runtime_lb} justifies the exponential running time of the provers in \Cref{thm:rlp_zo,thm:rlp_gen}. %

\subsection{Verification with labeled samples}
\label{sec:lb-queries-needed}

In this section we prove a lower bound on the number of labeled samples needed for verification in the two-prover setting when the verifier only has access to $f$ via labeled samples (instead of queries). 

\begin{theorem}[Refereed learning with labeled samples]
    \label{thm:lower_bound_samples}
    Fix a representation of functions.
    Let $c>0$ be a sufficiently small absolute constant, and fix range $\cY=\zo$. For all $b\in\zo$ let $\acc_b(f,h_0,h_1,\cD)$ provide query access to $f,h_0,h_1$ and $Q_\cD$; and let $\acc_\cV(f,h_0,h_1,\cD)$ provide labeled samples $(x,f(x))$ where $x\sim \cD$, and query access to $h_0,h_1$, and $Q_\cD$. For all $d\in\N$, $\alpha\geq 1$ and $\eta\in(0,1)$, there exists a class of boolean functions $\cH$ and distributions $\mathds D$ such that every $(\alpha,\eta,1/3)$-refereed learning protocol for $\cH$ and $\mathds D$ with respect to $\lzo$ and oracles $\acc_0,\acc_1$ and $\acc_\cV$ requires verifier sample-complexity $\frac c\eta$. Moreover, the lower bound holds even if the representation of $h_0,h_1$, and $Q_\cD$ is given as input to all parties.
\end{theorem}

The intuition behind this proof is straightforward. Fix hypotheses $\sets{h_0,h_1}$, sample $b\sim\zo$ and let $f\gets h_b$. Now, consider a malicious prover that executes the honest protocol, except it ``pretends'' that $f=h_{1-b}$. As long as the verifier does not obtain any sample $x$ with $f(x)\neq h_{1-b}(x)$, it cannot refute the malicious prover's claim that $h_{1-b}$ has zero loss, and thus cannot determine if it should accept $h_0$ or $h_1$.

\begin{proof}
    We prove the lower bound when all parties receive as input a representation of $h_0,h_1$ and $Q_\cD$. 
    Set $\cH=\sets{h_0,h_1}$, where $h_0(x)=0$ for all $x\in\zo^d$, and $h_1(x)=0$ for all $x\neq0^d$ and $h_1(0^d)=1$. Set $\mathds D=\sets{\cD}$, where $\cD$ is the distribution that places probability mass $\eta$ on the point $0^d$ and is uniform otherwise.
    Suppose $\brackets{\cP_0,\cP_1,\cV}$ is an $(\alpha,\eta,\frac 13)$-refereed learning protocol for $\cH$ and $\mathds D$ with respect $\lzo$ and oracles $\acc_0,\acc_1$, and $\acc_\cV$. Let $\cD_f$ denote the distribution over $(x,f(x))$ where $x\sim\cD$.
    By \Cref{def:rlp_concrete} and the definition of $\acc$ and $\acc_\cV$, for all $b\in\zo$, $f=h_b$, and $\cP^*_{1-b}$ we have $\brackets{\cP^{f}_b(h_0,h_1,Q_\cD),\cP^*_{1-b}, \cV^{\cD_f}(h_0,h_1,Q_\cD)}=b$ with probability at least $\frac 23$. Since $h_0,h_1$, and $Q_\cD$ are fixed, we will not write them explicitly---that is, we will let $\cP^f_b$ denote $\cP^{f}_b(h_0,h_1,Q_\cD)$ and let $\cV^{\cD_f}$ denote $\cV^{\cD_f}(h_0,h_1,Q_\cD)$.

    For each $b\in\zo$ let the malicious prover $\cP^*_{1-b}$ execute the honest prover protocol $\cP^{h_{1-b}}$---that is, the honest prover protocol run as if the true function $f$ is $h_{1-b}$. At a high level, we will argue that if $b$ is sampled uniformly at random, then the verifier cannot distinguish between $b=0$ and $b=1$ until it samples the point $(0^d,f(0^d))$ from $\cD_f$, and thus cannot correctly output $b$. 

    Now, for each $b\in\zo$ and $f\gets h_b$, define the \emph{view of the verifier} $\view{\cV^{\cD_f}}_b$ as the distribution over the $m$ samples $(x,f(x))$ drawn from $\cD_f$, and the transcripts $T_b$ and $T_{1-b}$ between $\cV$ and $\cP^{f}_b$, and between $\cV$ and $\cP^*_{1-b}$. Let $E$ be the event that one of the $m$ samples drawn by the verifier is $(0^d,f(0^d))$. Since $\cP^*_{1-b}$ executes $\cP^{h_{1-b}}_{1-b}$ and the honest prover executes $\cP^{f}_b=\cP^{h_b}_b$, the distribution of ($T_0,T_1)$ is independent of $b$ (the verifier always interacts with $\cP^{h_0}_0$ and $\cP^{h_1}_1$), and thus $\view{\cV^{\cD_f}\mid \overline E}_0 = \view{\cV^{\cD_f}\mid \overline E}_1$. Next, we utilize the following fact from \cite{RaskhodnikovaS06}.
    
    \begin{fact}[Claim 4 \cite{RaskhodnikovaS06}]
    \label{fact:RS06_fact}
        Let $E$ be an event that happens with probability at least $1-\delta$ under the distribution $\cD$. Then $\dtv\paren{\cD|_E,\cD}\leq \delta'$, where $\delta'=\frac{\delta}{1-\delta}$.
    \end{fact}
    
    Applying the triangle inequality twice yields
    \begin{align*}
    \dtv\paren{\view{\cV^{\cD_f}}_0, \view{\cV^{\cD_f}}_{1}}
    & \leq  \dtv\paren{\view{\cV^{\cD_f}}_0, \view{\cV^{\cD_f}\mid \overline E}_{0}}\\
    & + \dtv\paren{\view{\cV^{\cD_f} \mid \overline E}_0, \view{\cV^{\cD_f} \mid\overline E}_{1}}\\
    & + \dtv\paren{\view{\cV^{\cD_f} \mid\overline E}_1, \view{\cV^{\cD_f}}_{1}}.
    \end{align*}
    Since $\Pr_{X\sim\cD_d}\brackets{X=0^d}=\eta$, the event $E$ occurs with probability at most $m\cdot \eta$. Combined with \Cref{fact:RS06_fact} and the fact that $\view{\cV^{\cD_f}\mid \overline E}_0 = \view{\cV^{\cD_f}\mid \overline E}_1$, we see that $\dtv\paren{\view{\cV^{\cD_f}}_0, \view{\cV^{\cD_f}}_{1}}<\frac 13$ whenever $m\leq \frac c\eta$ for a sufficiently small absolute constant $c>0$. The rest of the proof follows from standard arguments. Observe that
    \begin{align*}
    \Pr_{\substack{b\sim\zo\\ f\gets h_b}}\brackets{\brackets{\cP^{f}_b, \cP^*_{1-b}, \cV^{\cD_f}} = b} 
    & = \frac12\BBparen{\Pr_{\substack{f\gets h_0 \\ s\sim\view{\cV^{\cD_f}}_0}}\brackets{\cV(s)=0} + \Pr_{\substack{f\gets h_1\\ s\sim\view{\cV^{\cD_f}}}_1}\brackets{\cV(s)=1}}\\
    & = \frac12 + \frac12\BBparen{\Pr_{\substack{f\gets h_0 \\ s\sim\view{\cV^{\cD_f}}_0}}\brackets{\cV(s)=0} - \Pr_{\substack{f\gets h_1\\ s\sim\view{\cV^{\cD_f}}_1}}\brackets{\cV(s)=0}}\\
    &\leq \frac12 + \frac12\cdot\dtv\paren{\view{\cV^{\cD_f}}_0,\view{\cV^{\cD_f}}_1} < \frac 23.
    \end{align*}
    Since this contradicts the definition of an $(\alpha,\eta,\frac13)$-refereed learning protocol, we must have verifier sample complexity $m\geq \frac c\eta$. 
\end{proof}

\subsection{Verification without query access to the PMF}
\label{sec:dist_knowledge}

In this section, we prove that non-trivial refereed learning requires query access to the probability mass function $Q_\cD$ of the underlying distribution $\cD$. Specifically, we show that a verifier with query access to $f$, but only sample access to $\cD$, will require many samples from $\cD$.
(Note, however, that our refereed query delegation protocol in \Cref{sec:ref-query-del} means that the verifier only needs a single query to $Q_\cD$ to be efficient.) This is in contrast to the lower bound of \Cref{thm:lower_bound_samples}, where the distribution is known to be uniform, but the verifier is only given labeled samples $(x,f(x))$ and cannot query $f$. 

\begin{theorem}[Refereed learning without query access to $Q_\cD$]
    \label{thm:dist_knowledge}
    Fix a representation of functions. Let $c>0$ be a sufficiently small constant, and fix range $\cY=\zo$. For all $b\in\zo$ let $\acc(f,h_0,h_1,\cD)$ provide query access to $h_0,h_1,f$, and $Q_\cD$; and let $\acc_\cV(f,h_0,h_1,\cD)$ provide query access to $h_0,h_1$, and $f$, and samples $x\sim \cD$.
    For all $d\in\N$, $\alpha\geq 1$, and $\eta\in(0,1)$, there exists a class of boolean functions $\cH$ and distributions $\mathds D$ such that every $(\alpha,\eta,1/3)$-refereed learning protocol for $\cH$ and $\mathds D$ with respect to $\lzo$ and oracles $\acc_0,\acc_1$, and $\acc_\cV$ requires verifier sample complexity $\frac c\eta$. Moreover, the lower bound holds even if the representation of $h_0$, $h_1$, and $f$ is given as input to all parties. 
\end{theorem}

\begin{proof}
    We prove the lower bound when all parties receive as input the representation of $h_0$, $h_1$, and $f$. 
    The proof is similar to the proof of \Cref{thm:lower_bound_samples}, except instead of choosing the function $f$ to be either $h_0$ or $h_1$, we randomly select a distribution $\cD_0$ or $\cD_1$ such that $\cL_{\cD_b}(h_b, f) = 0 < \alpha\cdot\cL_{\cD_b}(h_{1-b}, f)$ for all $b\in\zo$. Consider the family of functions $\cH = \sets{h_0,h_1,f}$ where $f(x)=0$ for all $x\in\zo^d$, $h_0(x)=0$ for all $x\neq 0^d$ and $h_0(0^d)=1$, and $h_1(x)=0$ for all $x\neq 1^d$ and $h_1(1^d)=1$. Next we define the family of distributions.
    For each $b\in\zo$, let $\cD_b$ place probability mass $\eta$ on the point $(1-b)^d$ and be uniform over $\zo^d\setminus \sets{b^d,(1-b)^d}$. Let $\mathds D=\sets{\cD_0,\cD_1}$. For simplicity, let $Q_b$ denote the PMF $Q_{\cD_b}$ for all $b\in \zo$. Notice that $\cL_{\cD_b}(h_b,f)=0$ since $f$ and $h_b$ agree everywhere except $x=b^d$, whereas $\cL_{\cD_b}(h_{1-b}, f) = \eta$ since $h_{1-b}$ and $f$ disagree on $x=(1-b)^d$ which is in the support of $\cD_b$.
    Thus, if $\brackets{\cP_0,\cP_1,\cV}$ is an $(\alpha,\eta,\frac 13)$-refereed learning protocol for $\cH$ and $\mathds D$ with respect to $\lzo$ and oracles $\acc_0$, $\acc_1$, and $\acc_\cV$, then for all $b\in\zo$ and all $\cP^*_{1-b}$ we have $\brackets{\cP^{Q_b}_b(f,h_0,h_1),\cP^*_{1-b},\cV^{\cD_b}(f,h_0,h_1)}=b$ with probability at least $\frac 23$.
    Since $h_0$, $h_1$, and $f$ are fixed we will omit them in the rest of the proof---that is, we let $\cP^{Q_b}_b$ denote $\cP^{Q_b}(f,h_0,h_1)$ and let $\cV^{\cD_b}$ denote $\cV^{\cD_b}(f,h_0,h_1)$.  

    Now, for all $b\in\zo$ let $\cP^*_{1-b}$ execute the honest prover protocol $\cP^{Q_{1-b}}_{1-b}$---that is, execute the protocol as if the distribution were $\cD_{1-b}$. Define the \emph{view of the verifier} $\view{\cV^{\cD_b}}_b$ as the distribution over query answers from $f$, samples $x_1,\dots, x_m\sim \cD_b$, and transcripts $T_b$ and $T_{1-b}$ from the interaction with $\cP^{Q_b}_b$ and $\cP^*_{1-b}$. Note that since $\cP^*_{1-b}$ executes the honest prover protocol $\cP^{Q_{1-b}}_{1-b}$, the verifier always interacts with $\cP^{Q_0}_0$ and $\cP^{Q_1}_1$, and thus the transcripts $T_0$ and $T_1$ are independent of $b$. Similarly, since the function $f$ is fixed in advance, the query answers are independent of $b$ as well.
    
    To complete the proof, we argue that $\cV$ cannot distinguish whether $b=0$ or $b=1$ until it draws many samples from $\cD_b$. Let $E$ be the event that one of the $m$ samples drawn by $\cV^{\cD_b}$ is either $0^d$ or $1^d$. Since the transcripts and query answers are independent of $b$, we have that $\view{\cV^{\cD_b}\mid \overline E}_0=\view{\cV^{\cD_b}\mid \overline E}_1$. Since $\cD_b$ places probability mass $\eta$ on $(1-b)^d$ and probability mass $0$ on $b^d$, %
    we have that $\Pr\brackets{E}\leq m\cdot \eta$. Applying \Cref{fact:RS06_fact}, and the same argument as in the proof of \Cref{thm:lower_bound_samples}, we obtain $\dtv\paren{\view{\cV^{\cD_0}}_0, \view{\cV^{\cD_1}
    }_1}<\frac 13$, whenever $m\leq \frac c\eta$ and thus $\Pr_{b\sim\zo}\brackets{\brackets{\cP^{Q_b}_b, \cP^*_{1-b}, \cV^{\cD_b}} = b}<\frac 23$. Since this contradicts the definition of an $(\alpha,\eta,\frac 13)$-refereed learning protocol, $\cV$ must draw $m\geq\frac c\eta$ samples from $\cD_b$. %
\end{proof}

\subsection{Prover time-complexity lower bound}
\label{sec:prover_runtime_lb}

In this section we show that any white-box refereed learning protocol can be used as a subroutine to decide if a 3-CNF formula is satisfiable. Formally, let SAT denote the set of satisfiable 3-CNF formulas with $d$ variables and $m$ clauses for all $d,m\in\N$, and suppose every algorithm that decides SAT with probability at least $\frac 23$ has runtime at least $T_{\text{SAT}}(d,m)$. In \Cref{thm:rlp_runtime_lb} we show that every white-box refereed learning protocol must have either prover or verifier runtime $\Omega\paren{T_{\text{SAT}}(d,m-1)/m}$, which justifies the running time of our protocols under standard complexity assumptions.

Throughout the section we let $\cH_{d,m}=\sets{\phi:\zo^d\to\zo}$ where $\phi$ is a $3$-CNF formula with $d$ variables and $m$ clauses.\footnote{Technically $\cH_{d,m}$ is a multiset, i.e., we include distinct representations of the same function as distinct elements.} Let $U_d$ be the uniform distribution on $\zo^d$. Additionally, %
define oracles $\acc_0=\acc_1=\acc_\cV=\acc$ by letting $\acc(f,\phi_0,\phi_1,\cD)$ provide $\phi_0$ and $\phi_1$, and query access to $f$. 

\begin{theorem}[Prover time-complexity lower bound]
    \label{thm:rlp_runtime_lb} 
     Fix $\alpha\in\N$. Suppose there exists a protocol $\brackets{\cP_0, \cP_1,\cV}$ that for all inputs $d,m\in\N$ is an $(\alpha,0,\frac 13)$-refereed learning protocol for $\cH_{d,m}$ and $\sets{U_d}$ with respect to $\lzo$ and oracles $\acc_0$, $\acc_1$, and $\acc_\cV$ (defined above). Then $\brackets{\cP_0,\cP_1,\cV}$ has either prover or verifier runtime $\Omega\paren{T_{\text{SAT}}(d,m-1)/m}$. %
\end{theorem}

\begin{proof}
    The main step in the proof is \Cref{claim:rlp_to_sat}, which states that an $(\alpha,0,\frac 13)$-refereed learning protocol for $\cH$ and $\sets{U}$ can be used to decide 3-SAT. 

    \begin{claim}[Reduction from 3-SAT]
        \label{claim:rlp_to_sat}
        Let $\brackets{\cP_0,\cP_1,\cV}$ be as in \Cref{thm:rlp_runtime_lb}. If the prover and verifier runtime is at most $T(d,m)$, then there exists an algorithm $\cA$ that decides 3-SAT with probability at least $2/3$ in time $O\paren{m\cdot T(d,m+1)}$. %
    \end{claim}
    \begin{proof}
       Below, we construct an algorithm $\cA$, which uses a refereed learning protocol as a subroutine to decide 3-SAT. Let $a>0$ be a sufficiently large constant.
        
        \algorithmtop{Algorithm $\cA$}{reduction from SAT}{alg:rlp_to_sat}{
        \textbf{Input:} 3-CNF formula $\phi$\\
        \textbf{Output:} accept/reject
        \begin{enumerate}
            \item Let $\phi_0(x) = \phi(x)\wedge (x_1)$ and $\phi_1(x) = \phi(x)\wedge (x_1\oplus 1)$.
            \item Repeat the following for each $j\in [a]$:
            \begin{enumerate}
                \item Sample $b_j\sim\zo$ uniformly and simulate $\brackets{\cP_0,\cP_1,\cV}$ by providing $\cP_0,\cP_1,$ and $\cV$ with query access to $f=\phi_{b_j}$ and input $\phi_0$ and $\phi_1$. 
                \item Let $\vb_j$ be the bit output by the verifier in the $\ord{j}$ simulation.
            \end{enumerate}
            \item If $\abs{\sets{j : \vb_j = b_j}}\geq \frac{7\cdot a}{12}$ then output \textbf{accept}; otherwise output \textbf{reject}.
        \end{enumerate}  
        }
        First, we argue that $\cA$ accepts with probability at least $\frac 23$ when $\phi$ is satisfiable. If $\phi$ is satisfiable, then at least one of $\phi_0$ or $\phi_{1}$ is satisfiable. Suppose $\phi_{0}$ is satisfiable. Then $\phi_{0}(x)\neq \phi_{1}(x)$ whenever $\phi_{0}(x)=1$. Moreover, since $f\in\sets{\phi_0,\phi_1}$ either $\phi_0$ or $\phi_1$ has zero loss (but not both), and hence the verifier $\cV$ must output $\vb_j = b_j$ with probability at least $\frac 23$. It follows by a Chernoff bound that for a sufficiently large absolute constant $a$, the verifier outputs $\vb_j = b_j$ on least $\frac{7\cdot a}{12}$ of the simulations with probability at least $\frac 23$,  and therefore $\cA$ will output accept with probability at least $\frac 23$.  

        Next, suppose $\phi$ is not satisfiable---that is, $\phi(x)=0$ for all $x\in\zo^d$. Then for all $b\in\zo$ formula $\phi_b$ is also unsatisfiable, and hence $\phi_b(x) = 0$ for all $x\in\zo^d$, and therefore $f(x)=0$ for all $x\in\zo^d$. Since $b_j$ is sampled uniformly at random, the probability that $\cV$ outputs $\vb_j=b_j$ is exactly $\frac 12$. It follows that for sufficiently large absolute constant $a$, the verifier outputs $\vb_j = b_j$ on least $\frac{7\cdot a}{12}$ of the simulations with probability at most $\frac 13$, and therefore $\cA$ will output reject with probability at least $\frac 23$. 

        Thus, if the protocol has verifier and prover time complexity $T(d,m)$, then, since answering each query made by the protocol requires evaluating $f=\phi_b$, a 3-CNF formula with $m$ clauses, algorithm $\cA$ runs in time at most $O(m\cdot T(d,m+1))$ and decides 3-SAT with probability at least $\frac 23$. 
    \end{proof}
    The proof of \Cref{thm:rlp_runtime_lb} follows since by definition of $T_{\text{SAT}}$ and \Cref{claim:rlp_to_sat} we must have $T(d,m)=\Omega(T_{\text{SAT}}(d,m-1)/m)$.
\end{proof}

\section{Applications and extensions of our protocols}
\label{sec:app_ext}

In this section we turn to applications and extensions of our refereed learning protocols. In \Cref{sec:juntas} we
show a natural setting (namely, where the hypothesis functions $h_0$ and $h_1$ are juntas) in which both the prover and the verifier can be implemented efficiently.
In this regime, the verifier's runtime is an arbitrary $\poly(d)$ factor smaller than the time required to solve this problem without the provers, showing that a refereed learning protocol can save the verifier significant computational resources even when the provers are computationally bounded.
In \Cref{sec:imprecise} we show how to extend our protocols, which deal with $\lambda$-precise loss functions and distributions, to handle loss functions and distributions specified to arbitrary precision.

\subsection{Efficient refereed learning for juntas}
\label{sec:juntas}

We now show how our protocol from \Cref{thm:rlp_zo} can be implemented efficiently for a natural class of hypotheses; moreover, we show that the verifier in this protocol takes time which is smaller by an arbitrary polynomial factor  than the time needed for this family of hypotheses without access to the provers.

For each $d,j\in\N$ let $\cH_{d,j} = \sets{h:\zo^d\to\zo \mid \text{$h$ is a $j$-junta}}$, and let $U_d$ denote the uniform distribution over $\zo^d$. Recall that, for $d,j\in\N$, a function $h:\zo^d\to\zo$ is a $j$-junta if there exists a set $J\subset [d]$ with $|J|\leq j$ and a function $g_h:\zo^J\to\zo$ such that $h(x)=g_h(x_J)$ for all $x\in\zo^d$, where $x_J = x_{J_1}x_{J_2}\cdots x_{J_j}$---that is, the value of $h(x)$ is uniquely determined by the setting of $x$ at the indices in $J$.

Now assume that, in addition to the usual query access to $h_0$ and $h_1$,  
the provers and verifier obtain  the junta indices $J_0$ and $J_1$ as input.  We show:

\begin{proposition}
    \label{prop:juntas}
    Let $\acc_0=\acc_1$ provide query access to $h_0$ and $h_1$, and let $\acc_\cV(f,h_0,h_1,\cD)$ provide query access to $f,h_0$, and $h_1$. There exists a protocol $\brackets{\cP_0,\cP_1,\cV}$ that, for all inputs $d,j\in\N$ and $\eps,\beta>0$, is a $(1+\eps,0,\beta)$-refereed learning protocol for $\cH_{d,j}$ and $\sets{U_d}$ with respect to $\lzo$ and $\acc_0,\acc_1,\acc_\cV$. Moreover, 
    for all  $h_0,h_1\in\cH_{d,j}$, if the junta bits $J_0$ and $J_1$ are given as input to all parties then the protocol has the following guarantees: 
    \begin{itemize}
        \item The verifier runs in time $(1+\frac{1}{\eps^2})\log\frac1\beta\poly d$ and makes $O\bparen{\paren{1 + \frac {1}{\eps^2}}\log\frac1\beta}$ queries to $f$. 
        \item The provers run in time $(1+\frac{1}{\eps^2})\log\frac1\beta\cdot 2^{2j}\poly d$.
        \item The communication complexity of the protocol is $(1+\frac{1}{\eps^2})\log\frac1\beta\poly d$.
    \end{itemize}
\end{proposition}

To prove this statement, we use the following theorem, which can be viewed as an efficient, special case of certifiable sample (\Cref{lem:cert_sample}) for sampling uniformly from some set $S\subseteq \zo^d$.

\begin{claim}[Certifiable index]
    \label{claim:cert_index}
    For all $d\in\N$ let $\prec_d$ be a total ordering on $\zo^d$. There exists a protocol $[\cP_0,\cP_1,\cV]$ such that for all $d\in\N$, all sets $S\subset \zo^d$ ordered according to $\prec_d$, and all $i\in [|S|]$ the following holds: 
    \begin{enumerate}
        \item For all $b\in\zo$ and all $\cP^*_{1-b}$ we have $\brackets{\cP^{\prec_d}_b(S), \cP^*_{1-b}, \cV^{S,\prec_d}}(d,i)=S_i$, the $\ord{i}$ element in $S$ (where the initial element in $S$ has index $1$).  %
        \item The runtime of the verifier and communication complexity of the protocol is $\poly d$.
    \end{enumerate} 
\end{claim} 

\begin{proof}[Proof of \Cref{claim:cert_index}]
    Let $\prec$ denote $\prec_d$. At a high level, our protocol works as follows: First, the verifier requests $\wh x = S_i$ from one of the provers, and using its membership query oracle to $S$, the verifier checks that $\wh x\in S$. Then, the verifier runs the certifiable sum protocol with $t(x) = \Ind
    \brackets{x\in S \wedge x\prec \wh x}$ to compute $s = \abs{\sets{x : x\prec \wh x}}$. If $s\neq i-1$ then the verifier rejects. Fix $b\in\zo$, and define the protocol $\brackets{\cP_0,\cP_1,\cV}_b$ as follows:
    \protocoltop{$\brackets{\cP^{\prec}_0(S),\cP^{\prec}_1(S),\cV^{S,\prec}}_b(d,i)$}{certifiable index}{prot:cert_index}{
    \begin{enumerate}
        \item $\cP_b$: Send $\wh{x} = S_i$ to $\cV$.
        \item \label{line:confirm_member} $\cV$: If $\wh{x}\not\in S$ then output reject. Otherwise, send $\wh x$ to $\cP_{1-b}$.
        \item \label{line:use_size} $\cV$: Provide query access to $A=\sets{x\in S : x\prec \wh{x}}$\footnote{The verifier need not construct $A$ explicitly} using oracle for $S$ and $\prec$. Execute certifiable sum (\Cref{lem:cert_sum}) with $t = \Ind\brackets{x\in A}$ and $\lambda\gets d$. Accept if $\sum_{x\in\zo^d} t(x) = |A| = i-1$ and reject otherwise.
    \end{enumerate}
    }

    Without loss of generality assume $\cP_0$ is honest. Then $\wh x = S_i$ and hence $|A| = i-1$. By \Cref{lem:cert_sum}, protocol $\brackets{\cP_0,\cP^*_1,\cV}_0$ outputs accept. Next, consider $\brackets{\cP_0,\cP^*_1,\cV}_b$. If $\cP^*_1$ sends $\wh x\neq S_i$, then, if $\wh x\not \in S$, the protocol rejects in Step \ref{line:confirm_member}. On the other hand, if $\wh x\in S$ then $|A| = |\sets{x\in S\mid x\prec \wh x}\neq i-1$. By \Cref{lem:cert_sum}, the sum $\sum_{x\in\zo^d}t(x) = |A|\neq i-1$, and hence the protocol outputs reject.   
    To complete the proof, define the protocol $\brackets{\cP_0,\cP_1,\cV}$ as follows: for each $b\in\zo$ run protocols $\brackets{\cP_0,\cP_1,\cV}_b$, and return the $\wh x$ from the first round of an accepting execution. By the above argument, the protocol $\brackets{\cP_0,\cP_1,\cV}$ always outputs $S_i$. The communication complexity and runtime guarantees follow by inspection of \Cref{prot:cert_index} and from the guarantees of \Cref{lem:cert_sum}.
\end{proof}

\begin{proof}[Proof of \Cref{prop:juntas}]
    At a high level, we show that the provers can efficiently compute the set $S=\sets{x\mid h_0(x)\neq h_1(x)}$. Moreover, since the distribution is uniform, the verifier can use certifiable sum (\Cref{lem:cert_sum}) and certifiable index (\Cref{claim:cert_index}) to efficiently sample a uniform element of $S$. To argue that the provers can also execute these protocols efficiently we leverage the junta structure of $h_0$ and $h_1$.

    Let $J_0,J_1\subset [d]$ be the set of junta bits for $h_0$ and $h_1$, respectively. In what follows, let $c>0$ be a sufficiently large, absolute constant. 
    
    \protocoltop{$\brackets{\cP^{h_0,h_1}_0,\cP^{h_0,h_1}_1,\cV^{f,h_0,h_1}}\paren{d,j,\eps,\beta,J_0,J_1}$}{refereed learning for juntas}{prot:rlp_junta}{
    \begin{enumerate}
        \item $\cP_0,\cP_1$: Let $J\gets J_0\cup J_1$. Query $h_0$ and $h_1$ on all settings of the bits in $J$. Let $S\gets\sets{x\in\zo^d \mid h_0(x)\neq h_1(x)}$, with a lexicographic ordering.\footnote{The provers need not explicitly compute $S$.} %
        \item $\cV$: Execute certifiable sum (\Cref{lem:cert_sum}) with $t(x)=\Ind\brackets{h_0(x)\neq h_1(x)}$ to obtain $|S|$.
        \item $\cV:$ Set $m\gets c\paren{1+\frac{1}{\eps^2}}\log\frac1\beta$. Sample $x_1,\dots,x_m\sim [|S|]$ and execute certifiable index (\Cref{claim:cert_index}) to obtain $x_1,\dots,x_m\gets S_{i_1},\dots, S_{i_m}$.   
        \item $\cV:$ Query $f$ on $x_{1},\dots, x_{m}$ and output $\vb=\arg\min_{s\in\zo}\abs{\sets{i\in[m] : h_s(x_i)\neq f(x_i)}}$.
    \end{enumerate}}

    We first argue that the provers in \Cref{prot:rlp_junta} can be implemented efficiently. Since $h_0$ and $h_1$ are $j$-juntas with junta bits $J_0$ and $J_1$,  the provers can compute $|S|$ and the $\ord{i}$ element of $S$ (ordered lexicographically) in time $O(2^{2j})$ by querying $h_0$ and $h_1$ on all $2^{2j}$ settings of the bits in $J_0$ and $J_1$. It follows that executing certifiable sum in the second step and certifiable index in the third step takes time at most $2^{2j}\poly d$. Thus, the provers run in time $(1+\frac{1}{\eps^2})\cdot 2^{2j}\poly d$.
    
    The soundness of the verifier follows by the same argument as in the proof of \Cref{thm:rlp_zo}. The runtime and query complexity of the verifier, and the communication complexity of the protocol follows by \Cref{lem:cert_sum,claim:cert_index}.   
\end{proof}

\paragraph{Comparison to the case of a proverless learner.}

Below, we argue that any learner that does not use the help of the provers must run in time $\Omega(2^j)$. Hence, for $j = c \log d$, the runtime of a proverless learner is an arbitrary $\poly d$ factor worse than the runtime of the learner with provers. Moreover, for this setting of $j$ the provers' runtime is still $\poly d$.
To see this lower bound on the runtime of the proverless learner, consider the following simple argument: Let $J=[j]$ be the junta bits and let $h_0$ be some fixed $j$-junta---that is, $h_0(x)=g_{h_0}(x_J)$ for some function $g_{h_0}:\zo^{J}\to\zo$. Now, choose a random $j$-junta $h_1$ as follows: sample $z\sim\zo^J$, let $g_{h_1}(z)=1-g_{h_0}(z)$ and $g_{h_0}=g_{h_1}$ otherwise, and let $h_1(x)=g_{h_1}(x_J)$. Notice that $h_0$ and $h_1$ are each $j$-juntas\footnote{There is an annoying but fixable issue that can arise here where $h_1$ may be a $(j-1)$-junta. To remedy this, we can simply choose a $g_{h_0}$ that cannot be made into a $(j-1)$-junta by flipping $g_{h_0}(x)$ on a single $x\in\zo^j$. For example, setting $g_{h_0}$ at random for each $x\in\zo^J$ ensures this is satisfied with high probability.} with junta bits $J$, and agree everywhere except on points $x$ with $x_J=z$. Now, choose $b\sim\zo$ uniformly, set $f\gets h_b$, and consider an algorithm $\cA$ that, given input $J$ and query access to $h_0$, $h_1$, and $f$, outputs $b$ with probability at least $\frac 23$. Notice that before $\cA$ queries one of the functions on a point $x$ such that $h_0(x)\neq h_1(x)$, the query answers are independent of $b$. Since $h_0(x)\neq h_1(x)$ if and only if $x_J=z$, and since $z$ is chosen uniformly at random, $\cA$ will require $\Omega(2^j)$ queries to output $b$ with probability at least $\frac 23$. 

\paragraph{Only the provers need to know $J_0$ and $J_1$.}
Finally, we note that  \Cref{prop:juntas} can be readily extended to the setting where only  provers are provided $J_0$ and $J_1$ as input, via  the following simple protocol: for each $b\in\zo$ and $i\in J_b$, the provers send $i$ as well as points $x$ and $x'=x^{\oplus i}$ ($x$ with the $\ord{i}$ bit flipped) such that $h_b(x)\neq h_b(x')$. The verifier can then query $h_b$ and, if $h_b(x)\neq h_b(x')$, add $i$ to the set $\wh J_b$. Since $h_b$ is a junta with indices $J_b$, only indices $i\in J_b$ will be added to $\wh J_b$ by the verifier. Moreover, since at least one prover is honest and receives $J_0$ and $J_1$ as input, every $i\in J_b$ will be added to $\wh J_b$ by the verifier. Thus, the verifier will obtain $\wh J_b=J_b$ for each $b\in\zo$. Since finding such $x$ and $x'$ for each $i\in J_b$ can be done by querying $h_b$ once for each of the $2^j$ settings of the bits in $J_b$, the provers run in time $O(2^j)$. Similarly, since checking each candidate $i$ and pair $x$ and $x'$ received from the provers takes $2$ queries the verifier runs in time $O(j)$.

\subsection{\texorpdfstring{When $\cD$ and $\ell$ are arbitrarily precise}{When D and ell are arbitrarily precise}}
\label{sec:imprecise}

We now explain how our protocols can be implemented for distributions $\cD$ and metrics $\ell$ that are not $\lambda$-precise; however, in order to bound communication complexity, we still require $\ell$ to be \emph{$M$-bounded}---that is, $\ell(y,y')\leq M$ for all $y,y'\in \cY$. 
Specifically, we analyze the cost of rounding $Q_\cD$ and $\ell$ so that they are $\lambda$-precise. Combined with the protocols in \Cref{sec:rlp}, this allows us to design protocols with additive error $\eta$ for arbitrarily precise distributions and loss functions that only incur communication cost $\log\frac1\eta\poly d$ (there is no cost in query complexity).

Recall that by \Cref{claim:lambda-precise-distribution}, for all distributions $\cD$ over $\zo^d$, the distribution $\cD_\lambda$
(defined by $\cD_{\lambda}(x) = \frac{\floor{\cD(x)}_{\lambda}}{\sum_{x\in\zo^d} \floor{\cD(x)}_{\lambda}}$, where $\floor{y}_\lambda$ denotes $2^{-\lambda}\cdot\floor{2^{\lambda}\cdot y}$ for all $y\in\R$)
satisfies $\dtv(\cD,\cD_\lambda)\leq 2^{d+1-\lambda}$. Now, for metric $\ell$, let $\ell_\lambda(y,y')$ denote $\floor{\ell(y,y')}_\lambda$. Then, by definition of $\floor{\cdot}_\lambda$ we have $\abs{\ell_\lambda(y,y')-\ell(y,y')}\leq 2^{-\lambda}$. At a high level, the goal of this section is to explain how one can execute the protocols of \Cref{thm:rlp_zo,thm:rlp_gen} using $\cD_\lambda$ and $\ell_\lambda$, given an $M$-bounded metric $\ell$ and query access to $Q_\cD$, at the cost of a small additive error term $\eta$ and factor of $\log\frac1\eta$ in communication complexity and runtime.

In \Cref{prop:imprecise} we assume that, for any metric $\ell$, computing $\floor{\ell(y,y')}_\lambda$ takes unit time.

\begin{proposition}[Protocol for arbitrary $\cD$ and $M$-bounded $\ell$]
    \label{prop:imprecise}
     For each $b\in\zo$ let $\acc_b(f,h_0,h_1,\cD)$ provide query access to $h_0$, $h_1$, and $Q_\cD$, and let $\acc_\cV$ provide query access to $h_0,h_1,f$ and $Q_\cD$. Fix integer $M\in\N$, range $\cY$, and $M$-bounded metric $\ell$ on $\cY$.

    For all $d\in\N$, $\alpha\geq 1$, $\beta>0$, and $\eta\in(0,1)$, let $\lambda(d,\alpha,\beta,\eta) = d+\log\frac{\alpha M}{\eta}$ and let $\brackets{\cP'_0,\cP'_1,\cV'}_{d,\alpha,\beta,\lambda}$ be an $(\alpha,0,\beta)$-refereed learning protocol for $\fF$ and $\mathds D_\lambda$ with respect to $\ell_\lambda$ and  $\acc_0,\acc_1$, and $\acc_\cV$. Suppose $\brackets{\cP'_0,\cP'_1,\cV'}_{d,\alpha,\beta,\lambda}$ has communication complexity and verifier runtime $T(d,\alpha,\beta,\lambda)$ and makes $q(d,\alpha,\beta,\lambda)$ queries to $f$.   There exists a protocol $\brackets{\cP_0,\cP_1,\cV}$ that, for all  $d\in\N$, $\alpha\geq1,\beta>0$ and $\eta\in(0,1)$, is an $(\alpha,\eta,\beta)$-refereed learning protocol for $\fF$ and $\fD$ with respect to $\ell$ and $\acc_0,\acc_1,\acc_\cV$. Moreover, $\brackets{\cP_0,\cP_1,\cV}$ has verifier runtime and communication complexity $T(d,\alpha,\beta,\lambda) + \lambda\poly d$, and the verifier makes $q(d,\alpha,\beta,\lambda)$ queries to $f$.   
\end{proposition}

Combining \Cref{prop:imprecise} with \Cref{thm:rlp_zo,thm:rlp_gen} yields a $(1+\eps,\eta,\beta)$-refereed learning protocol for the zero-one loss, and a $(3+\eps,\eta,\beta)$-refereed learning protocol for any $M$-bounded metric loss. In comparison to the protocols of \Cref{thm:rlp_zo,thm:rlp_gen}, the new protocols work for any distribution, make the same number of queries to $f$, and only incur a cost of $d+\log\frac {\alpha M}{\eta}$ in the verifier runtime and communication complexity. In the remainder of the section we prove \Cref{prop:imprecise}.

\begin{proof}
    Let $\brackets{\cP_0,\cP_1,\cV}$ be the following protocol: 
    \protocoltop{$\brackets{\cP^{h_0,h_1,Q_\cD}_0,\cP^{h_0,h_1,Q_\cD}_1,\cV^{f,h_0,h_1,Q_\cD}}(d,\alpha,\eta,\beta)$}{Protocol for arbitrary precision $\cD$ and $\ell$}{prot:imprecise}{
    \begin{enumerate}
        \item $\cV$: Set $\lambda\gets 6d+\log\frac{\alpha M}{\eta}$. Execute certifiable sum (\Cref{lem:cert_sum}) with $t(x)=\floor{D}_\lambda$ and $\lambda\gets\lambda$ to obtain $T_\lambda = \sum_{x\in\zo^d} \floor{\cD(x)}_{\lambda}$.
        \item $\cV,\cP_0,\cP_1$: Using $\ell$, $T_\lambda$, and query access to $Q_\cD$, simulate $\brackets{\cP'_0,\cP'_1,\cV'}_{d,\alpha,\beta,\lambda}$ by providing access to $\ell_\lambda$ and query access to $Q_{\cD|_\lambda}$.
        \item Output the result of the simulation.
    \end{enumerate}
    }
    
    By \Cref{lem:cert_sum}, verifier $\cV$ correctly obtains $T_\lambda$, and hence can provide the required query access to $Q_{\cD_\lambda}$. Moreover, the overhead in runtime an communication cost is simply $\lambda\poly d$.
    
    It remains to show that if $\cL_\cD(h_{1-b},f\mid \ell) > \alpha\cL_\cD(h_b, f\mid \ell) + \eta$ for some $b\in\zo$, then $\brackets{\cP_0,\cP_1,\cV}$ outputs $b$ with probability at least $1-\beta$. We will argue that the loss functions $\cL_\cD(f,h\mid \ell)$ and $\cL_{\cD_\lambda}(f,h\mid \ell_\lambda)$ are close. Let $a(x) = \ell(f(x),h(x)) - \ell_\lambda(f(x),h(x))$ and let $b(x) = \cD(x)-\cD_\lambda(x)$ for all $x\in\zo^d$. Then by the definition of $\floor{\cdot}_\lambda$ and \Cref{claim:lambda-precise-distribution}, we have $|a(x)|\leq 2^{-\lambda}$ and $|b(x)|\leq 2^{d+1-\lambda}$, and thus, 
    \begin{align*}
        \abs{\cL_{\cD}(f,h\mid \ell) - \cL_{\cD_\lambda}(f,h\mid \ell_\lambda)}
        & = \abs{\sum_{x\in\zo^d}\ell(f(x),h(x))\cD(x) - \ell_\lambda(f(x), h(x))\cD_\lambda(x)}\\
        & = \abs{\sum_{x\in\zo^d}\ell(f(x),h(x))\cD(x) - (\ell(f(x), h(x))-a(x))(\cD(x)-b(x))}\\
        &= \abs{\sum_{x\in\zo^d} a(x)\cD(x) + \ell(f(x),h(x))b(x) - a(x)b(x)}\\
        &\leq \abs{a(x)} + 2^{d} M \abs{b(x)} + 2^d\abs{a(x)b(x)}\\
        &\leq 2^{-\lambda} + 2^{2d+1-\lambda}M + 2^{2d+1-2\lambda}\leq 2^{5d-\lambda}M.
    \end{align*}
    Let $r_b = \cL_{\cD}(f,h_b\mid \ell) - \cL_{\cD_\lambda}(f,h_b\mid \ell_\lambda)$ for each $b\in\zo$. If $\cL_\cD(h_{1-b},f\mid \ell)> \alpha\cL_\cD(h_b, f\mid \ell) + \eta$, then $\cL_{\cD_\lambda}(h_{1-b},f\mid \ell_\lambda) > \alpha\cL_{\cD_\lambda}(h_b, f\mid \ell) + \eta - \abs{r_{1-b}} - \alpha\abs{r_b}$. By the above reasoning we have $\abs{r_{1-b}} + \alpha\abs{r_b}\leq \alpha2^{6d-\lambda}M\leq \eta$, and hence $\eta - \abs{r_{1-b}} - \alpha\abs{r_b}\geq 0$. It follows that $\cL_{\cD_\lambda}(h_{1-b},f\mid \ell_\lambda)\geq \alpha\cL_{\cD_\lambda}(h_b, f\mid \ell)$, and therefore $\brackets{\cP'_0,\cP'_1,\cV'}_{d,\alpha,\beta,\lambda}$ will output $b$ with probability at least $1-\beta$.
\end{proof}

\section{Protocols with additive and mixed error}
\label{sec:add-mixed-error}

In this section, we describe two protocols for the additive error setting, i.e., the setting where $\eta > 0$. We restrict our focus to the setting with the zero-one loss function $\lzo$.

First, we briefly consider additive error in the single prover setting. That is, given query access to $h_0$, $h_1$, and $f$, a verifier $\cV$, with the help of a prover $\cP$, would like to decide which of $h_0$ and $h_1$ has better loss on $f$. Prior work of \cite{GoldwasserRSY21} gives a protocol for empirical risk minimization in the single prover setting that can be easily adapted to compare $h_0$ and $h_1$. At a high level, the protocol works as follows:\footnote{We describe a modified version of their protocol tailored to our setting. For a complete description of the protocol see the proof of Claim 5.2 (Simple Query Delegation) in \cite{GoldwasserRSY21}.}
(1) The verifier draws $\Theta\paren{\frac{1}{\eta^2}}$ unlabeled samples from $\cD$ and $\Theta\paren{\frac 1\eta}$ labeled samples $(x,f(x))$ where $x\sim\cD$, and sends the samples (not including labels) to the prover. (2) The prover labels the samples using $f$ and sends them back to the verifier. (3) The verifier checks that the prover's labels agree with its own set of labeled samples. If the labels disagree then the verifier rejects. Otherwise, the verifier uses the labels, along with query access to $h_0$ and $h_1$, to determine which of $h_0$ and $h_1$ achieves smaller loss on $f$. 

While the aforementioned protocol can be efficient for constant $\eta$, when $\eta$ is too small the requirements that the verifier draw $\frac 1\eta$ labeled samples and that the prover make $\frac{1}{\eta^2}$ queries to $f$ may be prohibitive. In \Cref{prop:add-queries,prop:mixed-error-rlp}, we show that in the two prover setting where the verifier has query access to $f$, one can achieve a considerably better dependence on $\eta$.
First, in \Cref{prop:add-queries}, we show that in the additive error setting $(\alpha=1,\eta>0)$, the verifier can replace the $\frac{1}{\eta}$ labeled samples with a single query to $f$. The protocol of \Cref{prop:add-queries} improves the efficiency of the verifier, but it still has provers that make $\frac{1}{\eta^2}$ queries to $f$.
In contrast, in \Cref{prop:mixed-error-rlp} we show that in the mixed additive/multiplicative error setting $(\alpha=1+\eps, \eta>0)$, the provers need only make $\frac{1}{\eps^2\eta} + \frac{1}{\eta}$ queries to $f$, and the verifier still only makes a single query to $f$. For constant $\eps$, this significantly improves the query complexity of the provers. 

\subsection{Additive error \texorpdfstring{$(\alpha=1,\eta > 0)$}{}}
\label{sec:add-err}

Below, we consider an \emph{additive-error guarantee} with $\alpha = 1$ and $\eta > 0$. We show a refereed learning protocol for this setting when the provers and verifier both have query access to $f$. In this protocol, both the prover and verifier are efficient.

At a high level, the protocol works as follows. The verifier draws $\frac{1}{\eta^2}$ \emph{unlabeled} samples from $\cD$ and executes refereed query delegation (\Cref{lem:cert_query}) to obtain their labels. The verifier outputs the hypothesis with smaller loss on the labeled sample.

\begin{proposition}[Additive error]
\label{prop:add-queries}
    For each $b\in\zo$ let $\acc_b(f,h_0,h_1,\cD)$ provide query access to $f$, and let $\acc_\cV(f,h_0,h_1,\cD)$ provide sample access to $\cD$ and query access to $f,h_0$ and $h_1$.
    Fix range $\cY$.
    There exists a protocol $\brackets{\cP_0,\cP_1,\cV}$ that, for all inputs $d\in\N$, and $\eta,\beta\in(0,1)$,
    is a $(1,\eta,\beta)$-refereed learning protocol for $\fF$ and $\fD$ with respect to $\lzo$ and oracles $\acc_0,\acc_1$, and $\acc_\cV$. The protocol has the following guarantees: %
    \begin{itemize}
        \item The verifier draws $O\bparen{\frac{1}{\eta^2}\log\frac1\beta}$ samples from $\cD$,
        has runtime $O\bparen{\frac{1}{\eta^2}\log\frac1\beta}$,
        and makes $1$ query to $f$.
        \item The provers make $O\bparen{\frac{1}{\eta^2}\log\frac1\beta}$ queries to $f$ and have runtime $O\bparen{\frac{1}{\eta^2}\log\frac1\beta}$.
        \item The protocol has communication complexity  $O\bparen{\paren{d+\log|\cY|}\cdot\frac{1}{\eta^2}\log\frac1\beta}$.
    \end{itemize}
\end{proposition}

\begin{proof}[Proof of \Cref{prop:add-queries}]
    We use the following protocol. Let $c>0$ be a sufficiently large absolute constant.
    \protocoltop{$\brackets{\cP^{f}_0,\cP^{f}_1,\cV^{f,h_0,h_1,\cD}}(d,\eta,\beta)$}{refereed learning with additive error}{prot:add}{
    \begin{enumerate}
        \item $\cV$:
        Let $m = \frac{c\log1/\beta}{\eta^2}$. Draw $m$ samples $x_1,\ldots,x_m \sim \cD$ and send them to $\cP_0$ and $\cP_1$. %
        \item $\cV$: Execute refereed query delegation (\Cref{lem:cert_query}) to simulate the protocol where the verifier queries $f$ on $(x_1,\ldots, x_m)$ and obtains $\sets{(x_i,f(x_i))}_{i\in[m]}$, using $1$ query to $f$. 
        \item $\cV$: Return $\vb\gets \arg\min_{b\in\zo}\abs{\sets{i\in[m] \mid  h_b(x_i)\neq f(x_i)}}$.
    \end{enumerate}
    }

    For each $b\in\zo$ let $p_b = \Pr_{x\sim\cD}[h_b(x)\neq f(x)]$, and assume without loss of generality that $p_{1-s}\geq p_{s} + \eta$ for some $s\in\zo$. We will show that $\vb=s$ with probability at least $1-\beta$. 

    Let $\wh p_b=\frac{1}{m}\abs{\sets{i\in[m] : h_b(x_i)\neq f(x_i)}}$.
    By Hoeffding's inequality, our choice of $m$, and the fact that $\Ex[\wh p_b] = p_b$ we have,
    \(
        \Pr\brackets{|\wh p_b-p_b| \geq \frac{\eta}{4}}
        \leq
        2\exp(-m\eta^2/8)<\beta.
    \)
    Thus, with probability at least $1-\beta$, we have $|\wh p_b-p_b|\leq \eta/4$ for each $b\in\zo$. Since we assumed $p_{1-s}\geq p_s+\eta$, this implies that $\wh p_s < \wh p_{1-s}$, and that $\cV$ outputs $\vb = s$ with probability at least $1-\beta$. The sample, communication, query, and time complexity guarantees follow from \Cref{lem:cert_query} and by inspection of \Cref{prot:add}.
\end{proof}

\subsection{Mixed additive and multiplicative error
\texorpdfstring{$(\alpha>1,\eta>0)$}{}}
\label{sec:mixed-err}

Below, we consider a \emph{mixed-error guarantee} with both $\alpha>1$ and $\eta > 0$. When the prover has query access to $f$, we construct a refereed learning protocol with efficient provers and an efficient verifier that only needs sample access to $\cD$ and a single query to $f$. In contrast to the setting where $\alpha=1$ where the prover makes $\frac {1}{\eta^2}$ queries to $f$ (see \Cref{prop:add-queries}), the prover in \Cref{prop:mixed-error-rlp} works for $\alpha=1+\eps$ and need only make $1+\frac{1}{\eps^2}$ queries to $f$. 

\begin{proposition}[Mixed error]
\label{prop:mixed-error-rlp}
    For each $b\in\zo$ let $\acc_b(f,h_0,h_1,\cD)$ provide query access to $f$, and let $\acc_\cV(f,h_0,h_1,\cD)$ provide sample access to $\cD$ and query access to $h_0$, $h_1$, and $f$.
    Fix range $\cY$. %
    There exists a protocol $\brackets{\cP_0,\cP_1,\cV}$ that, for all $d\in\N$, $\eps>0$ and $\eta,\beta\in(0,1)$, is a $(1+\eps,\eta,\beta)$-refereed learning protocol for $\fF$ and $\fD$ with respect to $\lzo$ and oracles $\acc_0,\acc_1$, and $\acc_\cV$. The protocol has the following guarantees: %
    \begin{itemize}
        \item The verifier draws $O\bparen{\paren{1+\frac{1}{\eps^2}}\cdot\frac{\log1/\beta}{\eta}}$ samples from $\cD$,
        has runtime $O\bparen{\paren{1+\frac{1}{\eps^2}}\cdot\frac{\log1/\beta}{\eta}}$,
        and makes $1$ query to $f$.
        \item The provers make $O\bparen{\paren{1+\frac{1}{\eps^2}}\log\frac1\beta}$ queries to $f$ and have runtime $O\bparen{\paren{1+\frac{1}{\eps^2}}\log\frac1\beta}$.
        \item The protocol has communication complexity  $O\Bparen{\bparen{d+\log|\cY|}\cdot\paren{1+\frac{1}{\eps^2}}\log\frac1\beta}$.
    \end{itemize}    
\end{proposition}

\begin{proof}[Proof of \Cref{prop:mixed-error-rlp}]
    Let $S = \sets{h_0(x)\neq h_1(x)}$. At a high level, the proof proceeds by arguing that the verifier can efficiently generate $\Theta\paren{1+\frac{1}{\eps^2}}$ unlabeled samples from $S$. The verifier can then execute refereed query delegation (\Cref{lem:cert_query}) to obtain labeled samples. By the same argument as in the proof of \Cref{thm:rlp_zo}, the verifier can determine which of $h_0$ or $h_1$ has better loss on $f$ (up to a multiplicative constant) without making any additional queries.

    Let $c>0$ be a sufficiently large absolute constant and consider the following protocol:
    \protocoltop{$\brackets{\cP^{f}_0,\cP^{f}_1,\cV^{f,h_0,h_1,\cD}}(d,\eps,\eta,\beta)$}{refereed learning with mixed error}{prot:mixed}{
    \begin{enumerate}
        \item $\cV$: Set $m=c\cdot\paren{\frac{2(2+\eps)}{\eps}}^2\log\frac1\beta$ and $t\gets \frac{2m}{\eta^2}$. Draw $t$ samples $x_1,\dots,x_t\sim\cD$, and let $x_1,\dots,x_m$ denote the first $m$ samples in $S=\sets{x\mid h_0(x)\neq h_1(x)}$.\footnote{The verifier need not construct $S$ to determine membership, and can instead query $h_0$ and $h_1$ for each $x_i$ in the sample.}
        If fewer than $m$ samples are in $S$ then output $\vb\sim\zo$.
        \item $\cV$: Execute refereed query delegation (\Cref{lem:cert_query}) to simulate the protocol where the verifier queries $f$ on $(x_1,\ldots, x_m)$ and obtains $\sets{(x_i,f(x_i))}_{i\in[m]}$, using $1$ query to $f$. 
        \item $\cV$: Return $\vb\gets \arg\min_{b\in\zo}\abs{\sets{i\in[m] : h_b(x_i)\neq f(x_i)}}$.
    \end{enumerate}
    }

    Let $\cL_b=\Pr_{x\sim\cD}\brackets{h_b(x)\neq f(x)}$ and assume that $\cL_1 > (1+\eps)\cL_0 + \eta$ (a symmetric argument suffices for the case when $\cL_0 > (1+\eps)\cL_1 + \eta$). We show that the verifier outputs $\vb=0$ with probability at least $1-\beta$. By the triangle inequality and the fact that $\cL_0\geq 0$ we have
    \begin{align}
    \Pr_{x\sim\cD}\brackets{h_1(x)\neq h_0(x)}\geq \cL_1-\cL_0 > \eta.
    \label{eq:diff_eta}
    \end{align}
    
    We first argue that after $t$ samples, the verifier obtains $x_1,\dots,x_m\sim\cD|_S$ with probability at least $0.9$. Let $K=\sum_{i\in[t]} \Ind\brackets{x_i\in S}$. By \eqref{eq:diff_eta} we have $\Ex\brackets{K} = \eta\cdot t$, and thus by Hoeffding's inequality, $\Pr\brackets{|K-\eta t|\geq \eta t/2}\leq 2\exp\paren{-\eta^2t/2}\leq \frac \beta2$, and hence, $K>\eta\cdot t/2\geq  m$ with probability at least $1-\frac\beta2$. %

    Next, we argue that if the verifier obtains $x_1,\dots,x_m\sim\cD|_S$, then $\vb = 0$ with probability at least $1-\frac\beta2$. By \Cref{lem:cert_query}, the verifier correctly obtains $\sets{(x_i,f(x_i)}_{i\in[m]}$. By \Cref{claim:distance_test}, since $\cL_1>(1+\eps)\cL_0$ we have $\Pr_{x\sim\cD|_S}\brackets{h_1\neq f(x)} > \frac 12 + \frac{\eps}{2(2+\eps)}$. As in the proof of \Cref{thm:rlp_zo}, if $\wh p = \frac1m\sum_{i\in[m]}\Ind\brackets{h_1(x_i)\neq f(x_i)}$, then by Hoeffding's inequality with $\delta = \frac{\eps}{2(2+\eps)}$ we have $\Pr\brackets{|\wh p - \Ex[\wh p] \geq \delta}\leq2\exp\paren{-2m\delta^2}$, which, by our setting of $m$, is at most $\frac\beta2$. Hence, $\wh p > \frac 12$ with probability at least $1-\frac\beta2$, and thus the verifier will output $\vb=0$. Combining the above two arguments we see that $\vb=0$ with probability at least $1-\beta$. The runtime, query complexity, and communication complexity follow from \Cref{lem:cert_query} and by inspection of \Cref{prot:mixed}.
\end{proof}

\addcontentsline{toc}{section}{Acknowledgments}
\section*{Acknowledgments}

We thank Mark Bun for simplifying the proof of \Cref{claim:gen-dist-test}. We thank Noam Mazor for significantly simplifying the certifiable sample protocol (\cref{lem:cert_sample}).

\addcontentsline{toc}{section}{References}
\ifnum\isjmlr=0
    \bibliographystyle{alphaurl}
\fi
\bibliography{ref}

\end{document}